%% file: arxiv_Feb.tex
\pdfoutput=1

\documentclass[11pt,a4paper]{article}
\usepackage[hyperref]{acl}
\usepackage{times}
\usepackage{inconsolata}
\usepackage{latexsym}
\usepackage{danudefs}
\usepackage{algorithmic}
\usepackage{algorithm}
\usepackage{color}
\usepackage{booktabs}
\usepackage{xspace}
\usepackage{url}
\usepackage{braket}
\usepackage[T1]{fontenc}
\usepackage[utf8]{inputenc}
\usepackage{graphicx}
\usepackage{microtype}
\usepackage{longtable}
\usepackage{multirow}
\usepackage{subcaption}
\usepackage{graphicx}
\usepackage[stable]{footmisc}
\usepackage{tabularx}
\usepackage{subcaption}
\usepackage{xurl}
\usepackage{array}
\usepackage{titlesec}
\usepackage{supertabular, booktabs}
\titlespacing*{\paragraph}{0pt}{0.5ex}{1em}
\usepackage[english]{babel}
\usepackage{hyperref}
\addto\captionsenglish{%
}
\addto\extrasenglish{%
}

\input{myacronyms}

\title{Map of Encoders -- Mapping Sentence Encoders using \\Quantum Relative Entropy}

\author{Gaifan Zhang \qquad \qquad Danushka Bollegala\\
  Department of Computer Science, University of Liverpool, UK\\
  {\tt \{sggzhan8,danushka\}@liverpool.ac.uk}}
\date{}

\begin{document}
\maketitle

\begin{abstract}
    We propose a method to compare and visualise sentence encoders at scale by creating a \textbf{map of encoders} where each sentence encoder is represented in relation to the other sentence encoders.
    Specifically, we first represent a sentence encoder using an embedding matrix of a sentence set, where each row corresponds to the embedding of a sentence.
    Next, we compute the \ac{PIP} matrix for a sentence encoder using its embedding matrix.
    Finally, we create a feature vector for each sentence encoder reflecting its \ac{QRE} with respect to a unit base encoder.
    We construct a map of encoders covering 1101 publicly available sentence encoders, providing a new perspective of the landscape of the pre-trained sentence encoders.
    Our map accurately reflects various relationships between encoders, where encoders with similar attributes are proximally located on the map.
    Moreover, our encoder feature vectors can be used to accurately infer downstream task performance of the encoders, such as in retrieval and clustering tasks, demonstrating the faithfulness of our map.
\end{abstract}

\section{Introduction}
\label{sec:intro}

\begin{figure}[t!]
    \centering
    \includegraphics[width=1\linewidth]{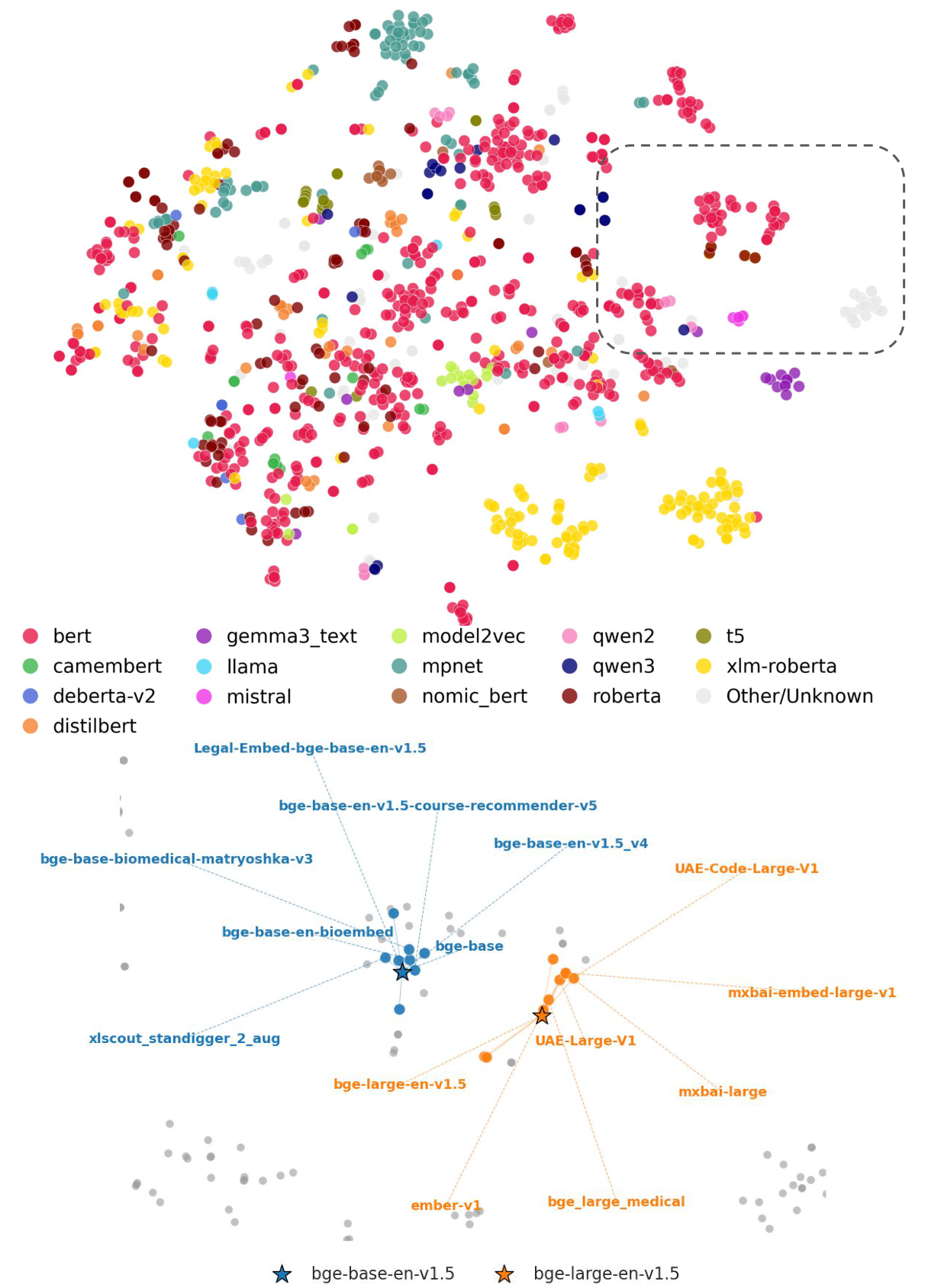} 
    
    \caption{Map of Encoders. Top: Map for 1101 sentence encoders, visualised by t-SNE and coloured by the encoder type. Bottom: A 30\% zoomed-in view of the dotted area showing the top-7 nearest neighbours for \href{https://huggingface.co/BAAI/bge-base-en-v1.5}{bge-base} and \href{https://huggingface.co/BAAI/bge-large-en-v1.5}{bge-large}. We see that the nearest neighbours belong to the same primary architecture and are closely located.}
    \label{fig:map-of-encoders}
\end{figure}

A large number of sentence encoders have been developed following different training algorithms \cite{Gao:2021c,Chen:2023a,Xu:2023a,Li:2020b} and fine-tuned on diverse tasks and datasets.
At the time of writing, 17,515 sentence encoders are published in \href{https://huggingface.co/models}{Hugging Face Hub}, and their global landscape remains unclear.
This is a significant blocker when improving or selecting sentence encoders for NLP applications.
Moreover, there is no universally agreed-upon metric for comparing sentence encoders.
For example, models are often grouped by attributes such as their parameter size, developer, pre-trained data source, fine-tuned tasks, etc. in the Hugging Face Hub, while leader-boards such as \href{https://huggingface.co/spaces/mteb/leaderboard}{MTEB} group models by their downstream task performance~\cite{muennighoff2023mtebmassivetextembedding}. 

To address this need, we propose a scalable and systematic method for comparing sentence encoders.
Unlike decoder-only \acp{LLM} that generate texts, which can be compared using likelihood scores~\cite{Oyama:2025a}, mapping sentence encoders is a significantly challenging task due to three main reasons: 
(1) \textbf{multi-variateness}: unlike the scalar likelihood scores that can be directly obtained from decoders, encoders return multi-variate embeddings, which require special care when comparing.
(2) \textbf{misalignment}: the vector spaces spanned by independently trained sentence encoders are often misaligned and incomparable.
(3) \textbf{dimensionality mismatch}: the dimensionalities of the encoders vary significantly from one another (i.e. in [384, 4096]) and are usually high.
Therefore, previously proposed visualisation methods for text-generating decoder \acp{LLM} cannot be applied directly to sentence encoders.

To fill this gap, we propose a method to map a given set of sentence encoders onto a shared two-dimensional space.
First, we represent each sentence encoder using an \emph{embedding matrix} computed from a fixed set of sentences.
Second, we convert the embedding matrix into a \ac{PIP} matrix, which acts as an encoder dimensionality-independent representation.
We further normalise the \ac{PIP} matrix to create a density matrix.
Third, we use \ac{QRE} to define a divergence measure between two encoders that simultaneously considers the eigenspaces of their density matrices, providing a scale, rotation and translation invariant information-theoretic comparator.
Using the spectral components in the \ac{QRE} computation, we represent all encoders by feature vectors in a shared feature space.
Finally, we use t-SNE~\cite{vandermaaten2008visualizing} and create a map that contains 1101 sentence encoders as shown in \autoref{fig:map-of-encoders}.

We see that encoders of varying parameter sizes, trained on different datasets and fine-tuned on diverse tasks, belonging to multiple primary architectures, are intuitively grouped in our map (\autoref{sec:results}), highlighting their attributes.
Moreover, we find a strong correlation (Spearman $>$ 0.8) between the actual downstream task performance of an encoder and that predicted using its feature vector for retrieval and clustering tasks (\autoref{sec:exp:downstream}).
This result implies that our map of encoders faithfully represents model performance and is practically useful for inferring attributes or performance of a novel encoder.

\section{Related Work}
\label{sec:related_work}

Numerous architectures have been proposed in prior work for creating sentence embeddings such as bi-directional MLM-based ones~\cite{devlin2019bertpretrainingdeepbidirectional, liu2019roberta, song2020mpnetmaskedpermutedpretraining} and autoregressive LLM-based ones~\cite{behnamghader2024llm2vec, wang2024improvingtextembeddingslarge, zhang2025qwen3embeddingadvancingtext}.
Moreover, learning algorithms significantly affect encoder performance~\cite{Gao:2021c, raffel2020exploring}.
Despite the growing complexity of the sentence encoder space, there is no systematic method to empirically quantify the relationships among different sentence encoders.

\citet{yin2018dimensionality} studied the relationships between the dimensionality and quality of \emph{word} embeddings.
They introduced \ac{PIP} matrices to represent the similarity between words, which contain geometric information about the word embedding space and are independent of the dimensionality of the embeddings. 
\citet{bollegala2022learningmetawordembeddings} used \ac{PIP} matrices to compare source embeddings for learning word-level meta-embeddings that aggregate properties from diverse source embeddings to create high-performing meta-embeddings on downstream tasks.
In our case, we use \ac{QRE} to measure the divergence between two encoders using their corresponding \ac{PIP} matrices over \emph{sentence} embeddings.
\citet{kornblith2019similarityneuralnetworkrepresentations} used \ac{PIP} matrices to measure pairwise representations of network layers.
They proposed Centered Kernel Alignment (CKA) by taking the trace (equivalent to sum over all eigenvalues) of the product of two \ac{PIP} matrices.
In contrast, we consider the complete eigenspaces of \ac{PIP} matrices, which better captures the geometry of the embedding spaces.

Quantum NLP has many applications, such as learning word interactions by modeling words in quantum states and representing sentences as density matrices \cite{li-etal-2018-quantum, wu-etal-2021-natural}, or detecting metaphor using density matrices to model the uncertainty and ambiguity of the literal meanings of words \cite{qiao-etal-2024-quantum}.
We represent the embedding space of a sentence encoder as a density matrix, extending this concept to compare sentence encoders.

\citet{umegaki1962conditional} defined a measure of divergence for density matrices by subtracting cross entropy from the negative von Neumann Entropy, which is now known as \ac{QRE}.
By representing encoders as density matrices derived from their \ac{PIP} matrices, we can use \ac{QRE} to measure the divergence between two given encoders.
\citet{De_Domenico_2016} extended the definition of \ac{QRE} to complex networks.
They used \ac{QRE} to measure information loss in network structures and perform multilayer clustering by minimising \ac{QRE}.
In contrast, we represent encoders as feature vectors where the axes correspond to the spectral components in the \ac{QRE} computation.

The most related to our motivation in this paper is the work by~\citet{Oyama:2025}, where they created a map of decoder \acp{LLM}.
Specifically, they proposed the use of log-likelihood scores computed for a fixed set of texts to represent an \ac{LLM}.
Subsequently, a map is created by projecting these log-likelihood vectors using t-SNE.
They proved that the squared Euclidean distance between the centred log-likelihood vectors approximates the \ac{KL} divergence between the corresponding \acp{LLM}.
Note that their map can be applied only for decoder models because it requires scalar log-likelihood scores, while we focus on sentence encoders producing multi-dimensional embeddings.
To the best of our knowledge, we are the first to create a map for sentence encoders.

\section{Mapping Sentence Encoders}
\label{sec:method}

Let us denote a set of $M$ sentence encoders by $\{f_m\}_{m=1}^{M}$, where each encoder $f_m$ embeds a given sentence $s$ to a $d_m$ dimensional vector, $\vec{f}_m(s)$.
Our goal is to represent all encoders in a common feature space that reflect their relationships.
As discussed in \autoref{sec:intro}, mapping sentence encoders is a challenging task due to their (1) multi-variate, (2) misaligned, and (3) varied dimensional embedding spaces. 
To obtain a dimensionality-independent and a shared representation for $f_m$, we first use a fixed set of $N$ sentences\footnote{As discussed in~\autoref{sec:sec:sentence_set}, $N = 10,000$ is sufficient.} $\cS$ ($=\{s_n\}_{n=1}^{N}$), and arrange their embeddings $\{\vec{f}_m(s_n)\}_{n=1}^{N}$ as rows in an embedding matrix $\mat{A}_m \in \R^{N \times d_m}$.
To address the challenges of disparate embedding spaces and mismatched dimensionalities, we use \ac{PIP} matrices that have been successfully used in prior work on meta-embedding learning~\cite{bollegala2022learningmetawordembeddings} to compare different embedding spaces.
\ac{PIP} computes the pairwise semantic similarities\footnote{When the embeddings are $\ell_2$ normalised, \ac{PIP} matrix is equivalent to the cosine similarity matrix between sentences.} among a fixed set of sentences according to a given encoder, which can be seen as capturing the semantic associations according to that encoder.

\subsection{Density Matrices of Encoders}
\label{sec:density_matrix}
The \ac{PIP} matrix of an encoder $f_m$, denoted as $\mat{G}_m$, is defined as the product of the embedding matrix and its transpose as in \eqref{eq:PIP}.
\begin{align}
    \label{eq:PIP}
    \mat{G}_m = \mat{A}_m \mat{A}_m\T \in \R^{N \times N}
\end{align}
\ac{PIP} matrices represent all sentence encoders in the same $N \times N$ real space, independently of their output embedding dimensionalities.
Note that $\mat{G}_m$ is \ac{PSD} with eigenvalues equal to the squared singular values of $\mat{A}_m$, and the left singular vectors of $\vec{A}_m$ become the eigenvectors of $\mat{G}_m$ (see  \autoref{sec:PIP_properties} for the proof).
Moreover, because $d_m \ll N$, \ac{PIP} matrices are often rank-deficient, having only $d_m$ non-zero eigenvalues at most.
Importantly, in practice we can efficiently compute the eigenspace of $\mat{G}_m$ ($N(N-1)/2$ elements) via the \ac{SVD} of $\mat{A}_m$ ($Nd_m$ elements), without having to explicitly compute $\mat{G}_m$, which otherwise requires $\O(N^2)$ memory and $\O(N^3)$ time complexity.

We normalise \ac{PIP} matrices by their trace to compute the corresponding density matrices as in \eqref{eq:density_matrix}.
\begin{align}
    \label{eq:density_matrix}
    \mat{\rho}_m = \frac{\mat{G}_m}{\Tr(\mat{G}_m)}
\end{align}
This normalisation ensures that $\Tr(\mat{\rho}_m)=1$.
As proven in~\autoref{sec:PIP_properties}, dividing a \ac{PIP} matrix by a positive scalar preserves its \ac{PSD} property.
Together, these two properties satisfy the requirements for $\mat{\rho}_m$ to be a density matrix~\cite{nielsen_chuang_2010}.

\subsection{Quantum Relative Entropy}
\label{sec:qre}
Although \ac{KL} divergence is a measure of the difference between two probability distributions~\cite{perez2008kullback}, which is used for the map of decoder LLMs by~\citet{Oyama:2025}, it is not applicable to sentence encoders in our case.
We compare the embedding spaces using density matrices, rather than probability distributions.
\ac{QRE} extends the definition of \ac{KL} divergence from probability distributions to density matrices and is a natural choice for comparing embedding spaces.

Given two embedding matrices $\mat{A}$ and $\mat{B}$, respectively with density matrices $\mat{\rho}$ and $\mat{\sigma}$, the \ac{QRE} of $\mat{\sigma}$ with respect to $\mat{\rho}$ is given by  \eqref{eq:QRE}.
\begin{align}
    \label{eq:QRE}
    S(\mat{\rho} \Vert \mat{\sigma}) = \Tr(\mat{\rho} \ln \mat{\rho}) - \Tr (\mat{\rho} \ln \mat{\sigma})
    % &= \Tr(\mat{\rho}(\ln \mat{\rho} - \ln \mat{\sigma}))
\end{align}
Here, $\ln$ denotes the matrix logarithm.\footnote{The logarithm $\ln \mat{\rho}$ is the matrix logarithm, defined as the scalar logarithm of the eigenvalues of $\mat{\rho}$, and not by the element-wise logarithm of the matrix.}
Because $\mat{\rho}$ and $\mat{\sigma}$ are symmetric \ac{PSD} matrices, they can be diagonalised into spectral representations by the Spectral Theorem~\cite{strang2022introduction} as proven in~\autoref{sec:PIP_properties}.
Therefore, we use eigendecomposition to overcome the numerical instability and the high computational costs involved with power series methods~\cite{higham2008functions} when computing the matrix logarithms in \eqref{eq:QRE}. 
Specifically, we decompose density matrices as in~\eqref{eq:spectral}.
\begin{align}
    \label{eq:spectral}
    \mat{\rho} = \sum_i^{K_{\rho}} \lambda_i \vec{v}_i \vec{v}_i\T, \quad \mat{\sigma} = \sum_j^{K_{\sigma}} \mu_j \vec{u}_j \vec{u}_j\T
\end{align}
where $\{(\lambda_i, \vec{v}_i)\}_{i=1}^{K_{\rho}}$ and $\{(\mu_j, \vec{u}_j)\}_{j=1}^{K_{\sigma}}$ denote the eigenpairs of $\mat{\rho}$ and $\mat{\sigma}$, with $K_{\rho}$ and $K_{\sigma}$ representing their respective ranks. 
The \ac{QRE} can then be expressed using the eigen components as in \eqref{eq:QRE_eigen}.
\par\nobreak
{\small
\vspace{-4mm}
\begin{align}
    \label{eq:QRE_eigen}
    S(\mat{\rho} \Vert \mat{\sigma}) = \sum_i^{K_{\rho}} \lambda_i (\ln \lambda_i) - \sum_i^{K_{\rho}} \lambda_i \left( \sum_j^{K_{\sigma}} \left( \vec{v}_i\T \vec{u}_j \right)^2 \ln \mu_j \right)
\end{align}
}%
Here, subject to the requirement that $\text{for any } i \text{, } \sum_{j=1}^{N} \left( \vec{v}_i\T \vec{u}_j \right)^2 = 1 $.
This requires that the two sets of eigenvectors of $\mat{\rho}$ and $\mat{\sigma}$ form complete orthonormal bases spanning $\R^N$.
See \autoref{sec:QRE-eigen} for the derivation of \eqref{eq:QRE_eigen} from \eqref{eq:QRE}.

In practice, the eigenspaces of two arbitrary encoders rarely align completely or span the same vector space, resulting in a loss of information of the corresponding eigenvector and $\sum_{j=1}^{N} \left| \vec{v}_i\T \vec{u}_j \right|^2 < 1$.
Theoretically, \ac{QRE} goes to infinity when the eigenspace of $\mat{\rho}$ is not contained in the eigenspace of $\mat{\sigma}$~\cite{nielsen_chuang_2010}. 
In our case, to maintain a finite and informative divergence, we treat the misaligned subspace as representing a high degree of divergence values where \ac{QRE} becomes large but finite.
To address this misalignment problem, we use \autoref{th:qre_mass} to approximate \ac{QRE}. 
\begin{theorem}
\label{th:qre_mass}
Let $\mat{\rho}$ be the density matrix of an encoder with non-zero orthonormal eigenpairs $\{(\lambda_i, \vec{v}_i)\}_{i=1}^{K_{\rho}}$. 
Let $\mat{\sigma}$ be the density matrix of a different encoder with non-zero orthonormal eigenpairs $\{(\mu_j, \vec{u}_j)\}_{j=1}^{K_{\sigma}}$. 
By perturbating $\mat{\sigma}$ with a small noise parameter $\epsilon > 0$ on its null space\footnote{In our experiments in~\autoref{sec:results}, we set $\epsilon = e^{-12}$.}, we can
approximate \ac{QRE} as follows:
\begin{align}
    \label{eq:qre-mass-proof}
    S(\mat{\rho} \Vert \mat{\sigma}_\epsilon) = \sum_{i=1}^{K_{\rho}} \lambda_i (\ln \lambda_i) - \sum_{i=1}^{K_{\rho}} \lambda_i (\mathcal{C}_i  + r_i \ln \epsilon)
\end{align}
where $c_i = \sum_{j=1}^{K_{\sigma}} (\vec{v}_i\T \vec{u}_j)^2$ is the captured mass, $r_i = 1 - c_i$ is the residual mass, and $\mathcal{C}_i = \sum_{j=1}^{K_{\sigma}} (\vec{v}_i\T \vec{u}_j)^2 \ln \mu_j$ is the cross-entropy contribution from the aligned subspace.
\end{theorem}

\autoref{th:qre_mass} enables the computation of \ac{QRE} values even when the eigenspaces are misaligned.
When the captured mass for a specific eigenvector of the encoder represented by $\mat{\rho}$ is less than 1, we quantify the missing information in this direction of the \ac{QRE} measurement using $\epsilon$.
In this situation, the contribution $- r_i \lambda_i\ln \epsilon$ to the \ac{QRE} tends to be large, where $r_i$ is the residual mass and $-\lambda_i \ln \epsilon$ represents the cross entropy between the eigenvector of the first encoder ($\mat{\rho}$) and the misaligned subspace of the second encoder ($\mat{\sigma}$).
See proof of \autoref{th:qre_mass} in \autoref{sec:QRE-mass}.

\begin{figure}[t!]
    \centering

    \includegraphics[width=0.8\linewidth]{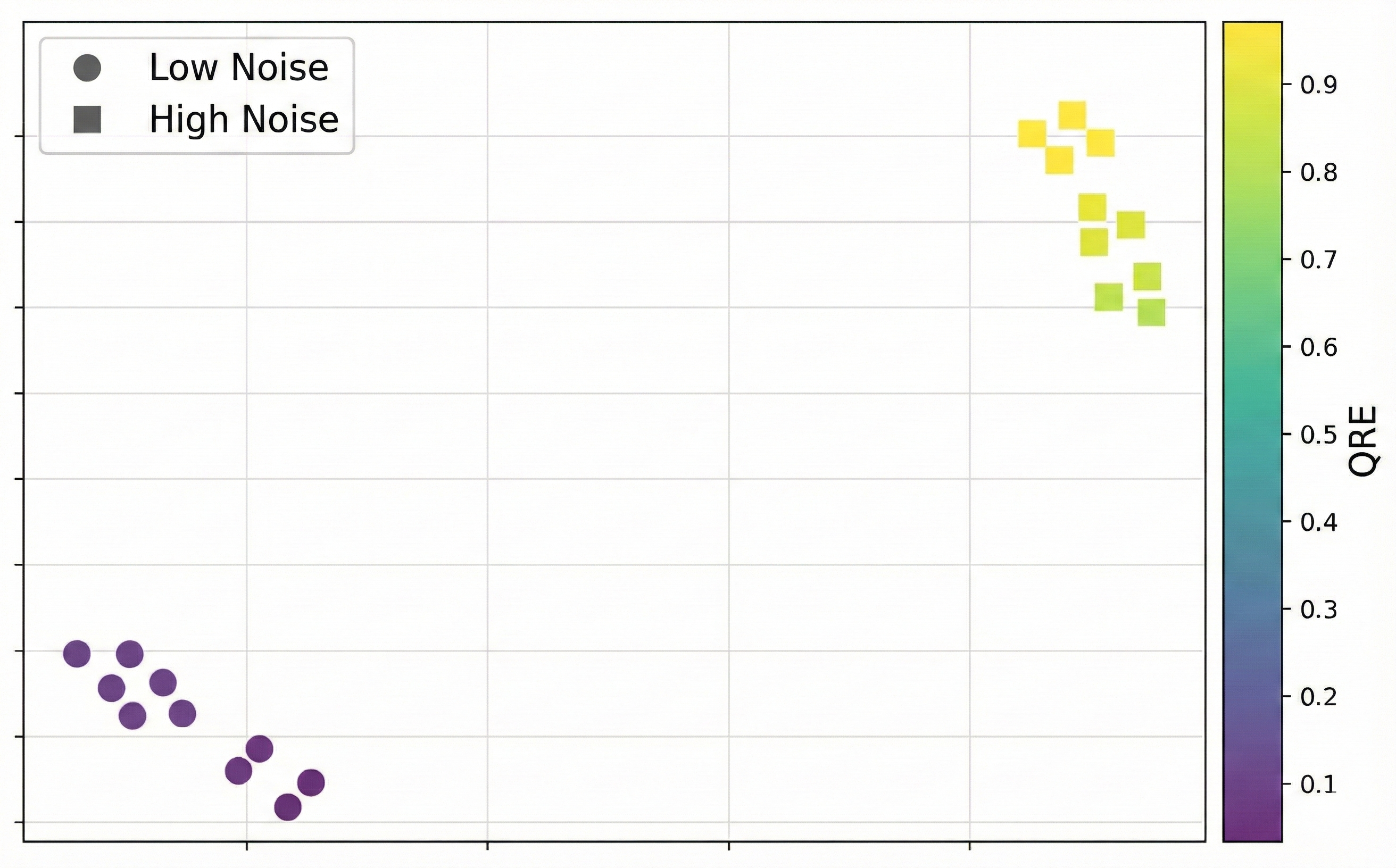}
     
    \caption{An illustrative example of the visualisation of \ac{QRE}-based feature vectors. We use the unit base encoder with 10,000 basis vectors and generate two groups of synthetic embeddings with low noise ($\sigma^2 \in [0,1]$) and high noise ($\sigma^2 \in [3, 4]$), respectively sampled from the normal distribution, $\mathcal{N}(0, \sigma^2 \mathbf{I})$. Each group has 10 embedding matrices. We use our proposed \ac{QRE} method to compute the feature vectors and visualise them using t-SNE. This example clearly shows that the low- and high-noise groups are separated by distinguishable \ac{QRE} values (sum of feature vectors), validating the faithfulness of our method. See \autoref{sec:illu} for details.}
    \label{fig:syn_map}
\end{figure}

\subsection{Representing Sentence Encoders}
\label{sec:encoder_embedding}

Recall that \ac{QRE} measures the divergence of an encoder relative to another encoder, which is problematic for two reasons:
(1) the pairwise computation of relative divergences between all encoders ($\cO(M^2)$) is expensive, and
(2) there is no shared feature representation common to all encoders.
To address this, we define a \emph{unit base encoder} $f_0$, an imaginary encoder that has a density matrix $\mat{\rho}_0$ where its eigenvectors form a complete orthonormal basis of $\R^{N}$.
% with unit basis vectors and all eigenvalues equal to $\frac{1}{N}$.
Note that we do not explicitly require the embedding (PIP or density) matrices for computing \ac{QRE} using~\autoref{sec:qre}, but only require the eigenspace of its density matrix.
We use the unit base encoder to represent each \emph{target encoder} as a \emph{feature vector} in a shared space.

Let $\{(\lambda_i, \vec{e}_i)\}_{i=1}^{N}$ denote the eigenpairs of the unit base encoder, where $\vec{e}_i$ is a unit basis vector with $1$ at the $i$-th dimension and $0$ elsewhere, and all $\lambda_i = \frac{1}{N}$.
This ensures that the eigenspace of the unit base encoder is \emph{isotropic}.
Additionally, this indicates that our unit base encoder corresponds to the eigendecomposition of the density matrix $\mat{\rho}_{0} = \frac{1}{N} \mat{I}$, where $\mat{I} \in \R^{N \times N}$ is the identity matrix.
Therefore, we can decompose the \ac{QRE} of the target encoder with respect to the unit base encoder into a feature vector, where each dimension corresponds to the divergence contribution measured along the specific basis vector. 

Specifically, we represent the $m$-th target encoder $f_m$ as an $N$-dimensional feature vector $\vec{\phi}_m$ as follows.
Let the density matrix corresponding to $f_m$ be $\mat{\rho}_m$ and its eigenpairs be $\{(\mu_j, \vec{u}_j)\}_{j=1}^{d_m}$.
The $w$-th dimension ($w \in [1, N]$) of $\vec{\phi}_m$ is denoted by $\phi_{m,w}$ and represents the contribution by the $w$-th eigen component of $\mat{\rho}_{0}$ to the calculation of \ac{QRE} in \eqref{eq:qre-mass-proof}.
Concretely, $\phi_{m,w}$ is given by \eqref{eq:axis_ele}.
\begin{align}
    \label{eq:axis_ele}
     \phi_{m,w} = \lambda_w \ln \lambda_w - \lambda_w (\mathcal{C}_w + r_w \ln \epsilon)
\end{align}
By construction, the sum of the elements in the feature vector $\vec{\phi}_m$ is equal to \ac{QRE} of the target encoder with respect to  the unit base encoder as expressed by \eqref{eq:feat-sum}.
\begin{align}
    \label{eq:feat-sum}
    \sum_{w=1}^N \phi_{m,w} = S(\mat{\rho}_{0} \Vert \mat{\rho}_m)
\end{align}

Note that we do \emph{not} select any single encoder as the comparison point, which makes the unit base encoder \emph{unbiased} and independent from the set of encoders that must be mapped.
Moreover, as explained above, all target encoders are represented in the same shared $N$-dimensional feature space.
These desirable properties guarantee that any novel encoder can be represented in the same map coordinates, without affecting the positions of the existing encoders.

\subsection{Projection onto a 2D Map}
\label{sec:t-sne}

Following the prior work on mapping decoder \acp{LLM} by~\citet{Oyama:2025}, we use \href{https://scikit-learn.org/stable/modules/generated/sklearn.manifold.TSNE.html}{t-SNE} to visualise the map of encoders.
t-SNE creates a map by capturing the implicit structures of high-dimensional data with low-dimensional manifolds, showing both global and local relationships in a two-dimensional map and is extensively used for data visualisation~\cite{li2016visualizing, gonzalez2022two}. 

As shown by \eqref{eq:feat-sum}, \ac{QRE} of a target encoder is given by the sum of elements in the corresponding feature vector.
Empirically, all feature values are positive due to the residual mass correction and the $\ell_1$ norm of $\vec{\phi}_m$,  $\norm{\vec{\phi}_m}_1$, is equal to $S(\rho_{0}||\rho_m)$ in practice.
Moreover, the $\ell_1$ distance between two target encoder feature vectors reflects the accumulated difference in cross-entropy arising from the distinct alignment of each encoder with the unit base encoder (see~\autoref{sec:feat-space} for proof).
Therefore, we use the $\ell_1$ distance measured between encoder feature vectors as the distance metric for finding nearest neighbours (in~\autoref{sec:results}) and for the t-SNE projections.

\autoref{fig:syn_map} shows the result of applying our proposed method to synthetic embeddings.
We see that the map perfectly separates the two groups of synthetic embeddings, thus accurately reflecting the inter- and intra-group relationships.

\section{Experiments}
\label{sec:exp}

The creation of feature vectors for the map of encoders requires a set of encoders $\{f_m\}_{m=1}^{M}$ and a set of sentences $\cS$ for creating the embedding matrices and their selection is described next.

\subsection{Selection of Sentence Encoders}
\label{sec:encoder_selection}

We use Hugging Face Hub to select sentence encoders that are compatible with the \href{https://huggingface.co/sentence-transformers}{sentence-transformers} library~\cite{huggingface_sentence_transformers}.
Adopting the model selection strategy by \citet{Oyama:2025}, we first select the most-downloaded 2000 sentence encoders (over the past 30 days, as of November 27, 2025).
Subsequently, we exclude the sentence encoders that cannot be correctly loaded using the \href{https://huggingface.co/sentence-transformers}{SentenceTransformer} class (see detailed discussion in~\autoref{sec:limitations}).
This filtering process produces a final set of 1101 sentence encoders for the map (see the full model list in \autoref{sec:full_model_list}).

\begin{figure}[t!]
    \centering
    
    % --- First Subfigure ---
    \begin{subfigure}[b]{\linewidth}
        \centering
        \includegraphics[width=\linewidth]{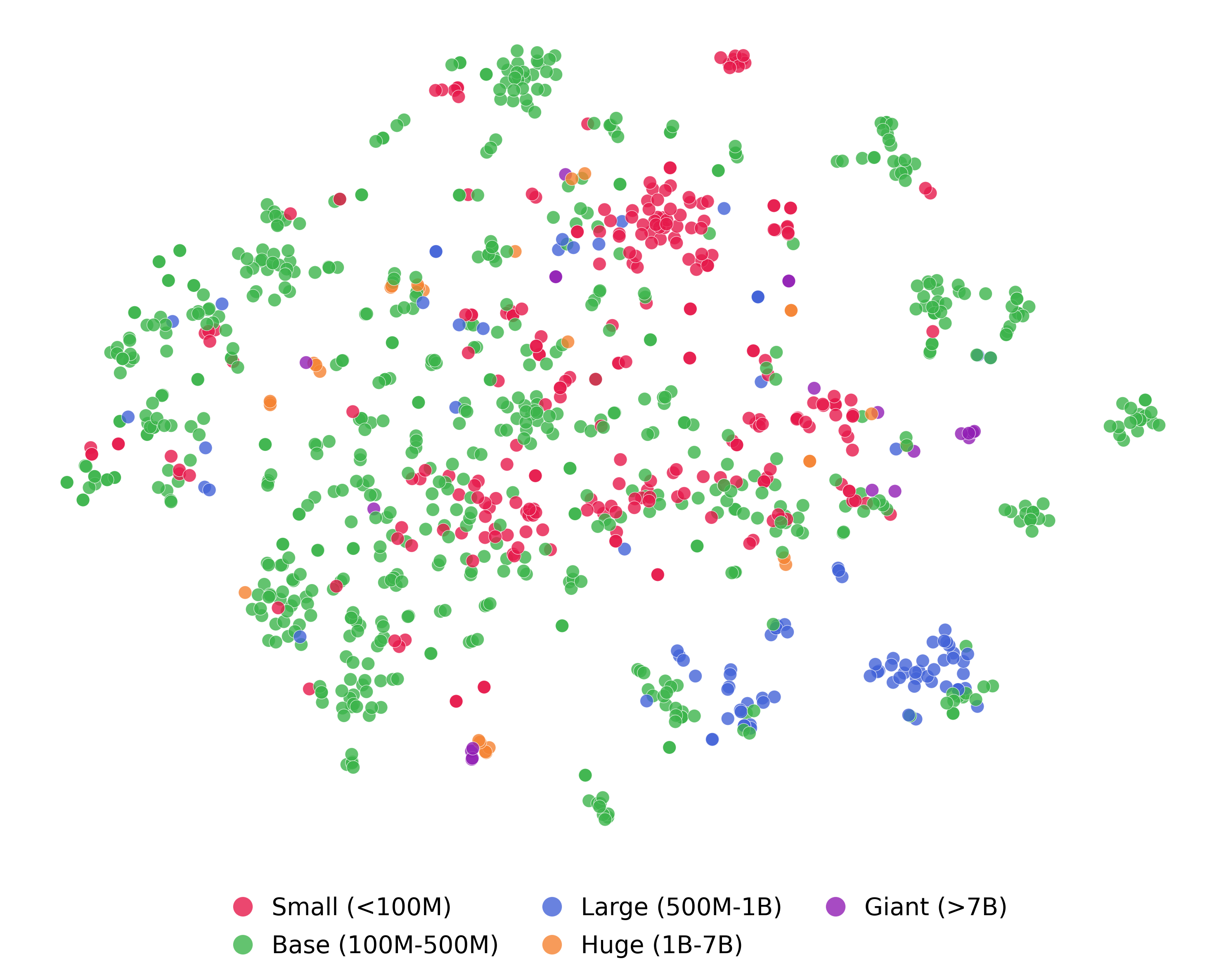}
        \caption{Encoder Parameter Size}
        \label{fig:map_size}
    \end{subfigure}
    
    \vspace{1em} % Adds space between figures
    
    % --- Second Subfigure ---
    \begin{subfigure}[b]{\linewidth}
        \centering
        \includegraphics[width=\linewidth]{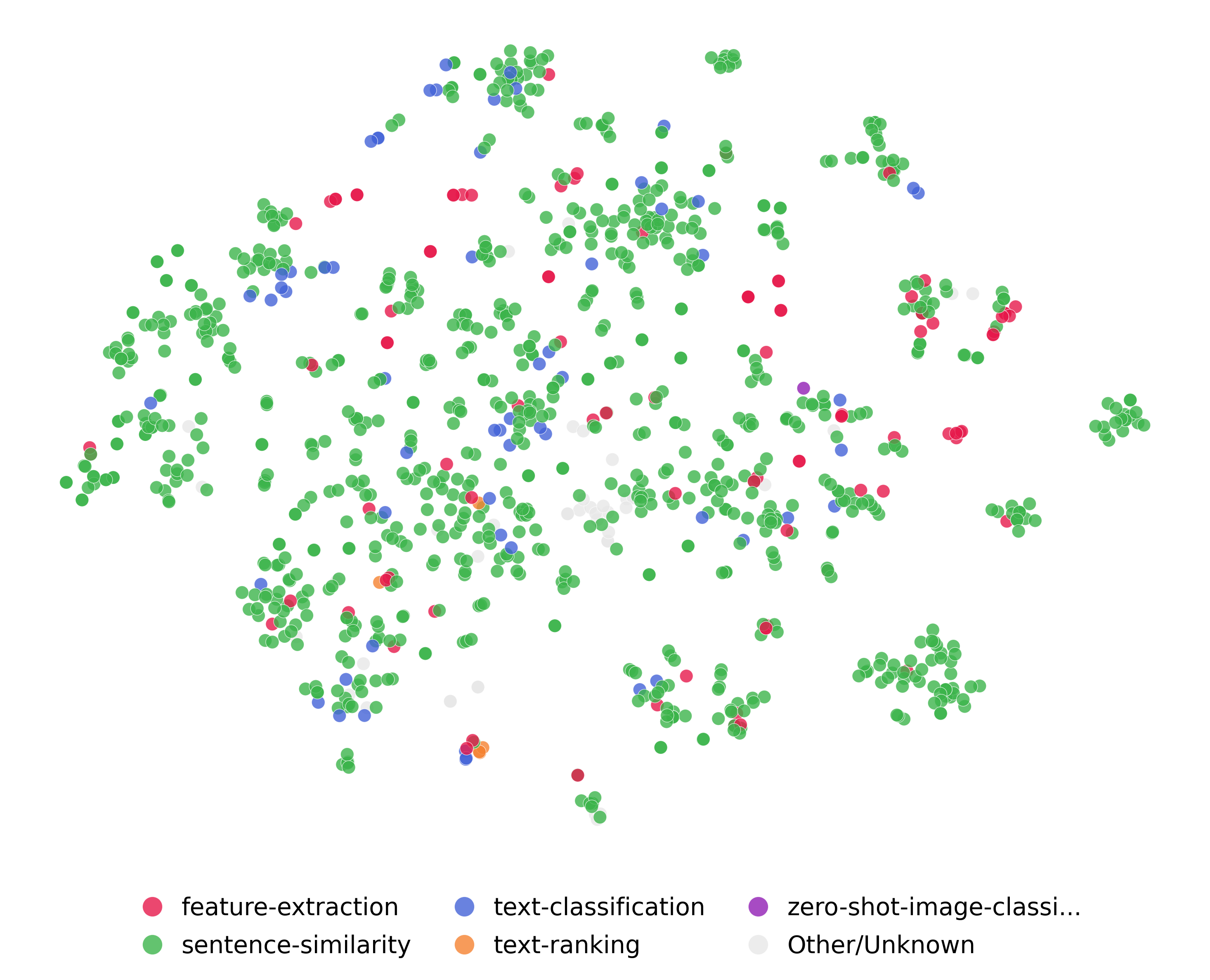}
        \caption{Training Task}
        \label{fig:map_task}
    \end{subfigure}
    
    \vspace{1em} % Adds space between figures
    
    % --- Third Subfigure ---
    \begin{subfigure}[b]{\linewidth}
        \centering
        \includegraphics[width=\linewidth]{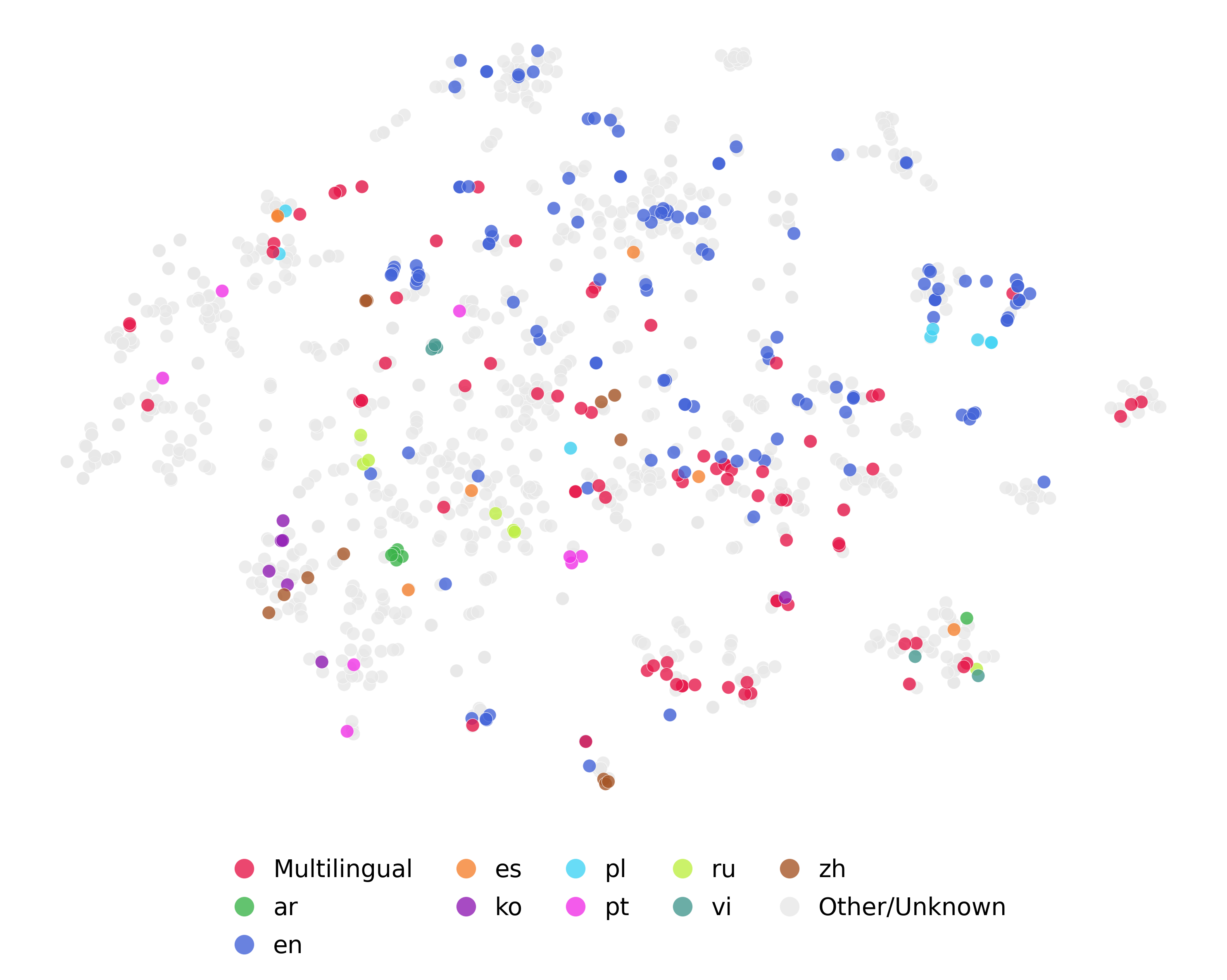}
        \caption{Pre-trained Language}
        \label{fig:map_language}
    \end{subfigure}

    % --- Main Caption and Label ---
    \caption{Maps by attributes.}
    \label{fig:maps_by_categories}
\end{figure}

\subsection{Selection of Sentence Set}
\label{sec:sec:sentence_set}

We randomly selected 10,000 high-quality and diverse sentences from a curated  \href{https://huggingface.co/datasets/agentlans/high-quality-english-sentences}{high-quality-English-sentences} \cite{agentlans2024highquality} dataset derived from the C4 \cite{raffel2023exploringlimitstransferlearning} and FineWeb \cite{penedo2024finewebdatasetsdecantingweb}.
This dataset is filtered by a text-quality assessment classifier, \href{https://huggingface.co/agentlans/deberta-v3-xsmall-quality}{deberta-v3-xsmall-quality}, to select a high quality set of sentences, covering a wide range of topics.
As discussed in~\autoref{sec:select_10000}, selecting 10,000 sentences is adequate to accurately represent all sentence encoders.

\begin{table}[t!]
\centering
\small
\resizebox{\columnwidth}{!}{%
\begin{tabular}{l c}
\toprule
\textbf{Encoder Name} & \textbf{$\ell_1$} \\
\midrule
\multicolumn{2}{l}{\textbf{Target: MPNet}} \\
 optimum-intel-internal-testing/all-mpnet-base-v2 & 0.0 \\
 flax-sentence-embeddings/all\_datasets\_v4\_mpnet-base & $3.22 e^{-4}$ \\
 spartan8806/atles-champion-embedding & 0.0429 \\
 wwydmanski/all-mpnet-base-v2-legal-v0.1 & 0.0440 \\
 hojzas/setfit-proj8-all-mpnet-base-v2 & 0.0628 \\
\midrule
\multicolumn{2}{l}{\textbf{Target: Multi-MPNet}} \\
 meedan/paraphrase-filipino-mpnet-base-v2 & 0.0398 \\
 sdadas/st-polish-paraphrase-from-mpnet & 0.0746 \\
 shtilev/medical\_embedded\_v5 & 0.0994 \\
 s-t/paraphrase-mpnet-base-v2 & 0.1083 \\
 projecte-aina/ST-NLI-ca\_paraphrase-multilingual-mpnet-base & 0.1087 \\
\midrule
\multicolumn{2}{l}{\textbf{Target: GIST-v0}} \\
 OrcaDB/gte-small & 0.0494 \\
 embaas/s-t-gte-small & 0.0494 \\
 thenlper/gte-small & 0.0494 \\
 MoralHazard/NSFW-GIST-small & 0.0698 \\
 JALLAJ/5epo & 0.0705 \\
\midrule
\multicolumn{2}{l}{\textbf{Target: bge-large-en}} \\
 katanemo/bge-large-en-v1.5 & 0.0 \\
 llmrails/ember-v1 & 0.0465 \\
 WhereIsAI/UAE-Large-V1 & 0.1284 \\
 ls-da3m0ns/bge\_large\_medical & 0.1290 \\
 OrcaDB/mxbai-large & 0.1447 \\
\bottomrule
\end{tabular}
}
\caption{Nearest neighbours for encoders sorted in ascending order of pairwise $\ell_1$ values. 
s-t denotes \textit{sentence-transformers}. 0.0 indicates the value is effectively zero at 12 decimal places.}
\label{tab:nearest_neighbors}
\end{table}

\section{Results}
\label{sec:results}

\subsection{Nearest Neighbours}
\label{sec:NNs}

We investigate the nearest neighbours of the encoders to assess whether our feature vectors can accurately capture the relationships between encoders.
We hypothesise that encoders derived from the same primary architectures would exhibit smaller pairwise $\ell_1$ values and cluster together within our map.
A zoomed-in map at the bottom of \autoref{fig:map-of-encoders} visualises the neighbours of the encoders \textit{bge-large} and \textit{bge-base} as an illustrative example, highlighting the proximity of neighbours with the same primary architecture in the map.

\autoref{tab:nearest_neighbors} shows the top 5 nearest neighbours in the ascending order of pairwise $\ell_1$ values for 4 randomly-selected target encoders (\href{https://huggingface.co/sentence-transformers/all-mpnet-base-v2}{MNPet}, \href{https://huggingface.co/sentence-transformers/paraphrase-multilingual-mpnet-base-v2}{Multi-MPNet}, \href{https://huggingface.co/avsolatorio/GIST-small-Embedding-v0}{GIST-v0}, \href{https://huggingface.co/BAAI/bge-large-en-v1.5}{bge-large-en}) (see the full list of nearest neighbours in~\autoref{sec:NNs_full_table}).
For all encoders, we inspect their model cards in Hugging Face Hub and have confirmed that the nearest neighbours are indeed fine-tuned versions of the corresponding target models.
In particular, from \autoref{tab:nearest_neighbors} we see that the encoders sharing the same name are fine-tuned from the target encoders.

For \textbf{MPNet}, \textit{all-mpnet-base-v2-legal-v0.1} is fine-tuned from MPNet on a legal dataset.
\textit{atles-champion-embedding} is also fine-tuned from MPNet on Semantic Textual Similarity (STS) tasks.
For \textbf{Multi-MPNet}, the neighbour \textit{medical\_embedded\_v5} is fine-tuned from Multi-MPNet on medical and clinical texts.
For \textbf{GIST-v0}, the neighbour, \textit{5epo}, is fine-tuned from \textit{BAAI/bge-small-en-v1.5}, where GIST-v0 is fine-tuned from the same encoder on a medical dataset.
The three neighbours \textit{OrcaDB/gte-small}, \textit{embaas/s-t-gte-small}, and \textit{thenlper/gte-small} have the same $\ell_1$ values, indicating that they are highly close to each other and possibly copies of the same encoder.
The \textit{gte-small} encoders share similarities in architectures with the \textit{bge-small} encoders, as both are constructed on small bert encoders with contrastive learning.
Finally, for \textbf{bge-large-en}, the nearest neighbours are copied or directly fine-tuned from bge-large encoders with additional curated datasets and training strategies.
We further investigate the hierarchical clustering of the nearest neighbours in~\autoref{sec:dendrogram_NN}, with different groups of neighbours clearly identifiable on the map.
This shows that our map captures both the local and global relationships between the encoders.

In conclusion, our method accurately represents encoders reflecting their close relationships, which might not be obvious from the encoder names alone, but are correctly captured in the encoder feature vectors and the map created therewith.

\subsection{Analysis of the Map}

We observe grouping patterns when the map is coloured by different attributes such as parameter size and encoder types (primary architecture of the model).\footnote{We extract six attributes from \href{https://huggingface.co/docs/huggingface_hub/main/en/package_reference/hf_api}{Hugging Face API}: encoder type, parameter size, training task, pre-trained language, training dataset, and dimensionality.}
\autoref{fig:map-of-encoders} visualises the encoders by encoder type.
Generally, the bert encoders are pervasive in the map, reflecting the wide usage of bert in sentence encoders.
The xlm-roberta encoders tend to group in the bottom right and middle left areas, while the mpnet encoders are close together at the top of the map.
The roberta and distilbert encoders are close to bert clusters, which is reasonable due to architectural similarity.
There are also clear clusters for specific encoder types, such as gemma3\_text and model2vec encoders.

From the map showing the encoder parameter sizes (\autoref{fig:map_size}), we see that encoders with smaller parameter sizes ($<$100M) tend to cluster in the middle, whereas encoders with larger parameter sizes (500M-1B) cluster on the bottom right side.
Encoders with parameter sizes ($>$1B) tend to be scattered across the map, instead of forming tight clusters.
This might be because smaller models are often distilled or fine-tuned from larger foundation models, making them similar to each other on the map.
For example, a closer look at the map reveals that the Qwen series encoders with different parameter sizes are closer to each other on the map (e.g. \href{https://huggingface.co/Qwen/Qwen3-Embedding-8B}{8B}, \href{https://huggingface.co/Qwen/Qwen3-Embedding-4B}{4B} and \href{https://huggingface.co/Qwen/Qwen3-Embedding-0.6B}{0.6B}).

From the map by the fine-tuning tasks (\autoref{fig:map_task}), we see that encoders are mostly fine-tuned on the sentence similarity tasks, a standard fine-tuning task for sentence encoders~\cite{reimers-2019-sentence-bert}, showing its prevalence in the map.
Feature extraction and text classification are also commonly used tasks for fine-tuning encoders~\cite{devlin-etal-2019-bert, wolf-etal-2020-transformers}.
Although they are not as prevalent as sentence similarity, they are also spread throughout the map.

Considering pre-trained languages (\autoref{fig:map_language}), encoders trained only on English (en) are mostly located on the top and right side of the map, whereas encoders trained on multilingual datasets are located at the centre and lower-right areas of the map.
Note that most of the multilingual datasets contain English.
Monolingual encoders (except English) are located mostly at the bottom of the map, such as for Arabic (ar) and Korean (ko). 
See \autoref{sec:more_map_class} for maps coloured by the dimensionality of the embeddings and the training datasets.

We further plot a dendrogram to visualise the hierarchical relationships between the top 100 most-downloaded encoders in~\autoref{fig:dendrogram_top_100}.
We see that encoders with exact/derived architectures are closely located (shown in the same colour), which is consistent with our map by encoder type (\autoref{fig:map-of-encoders}).

\begin{figure}[t!]
    \centering
    \includegraphics[width=1\linewidth]{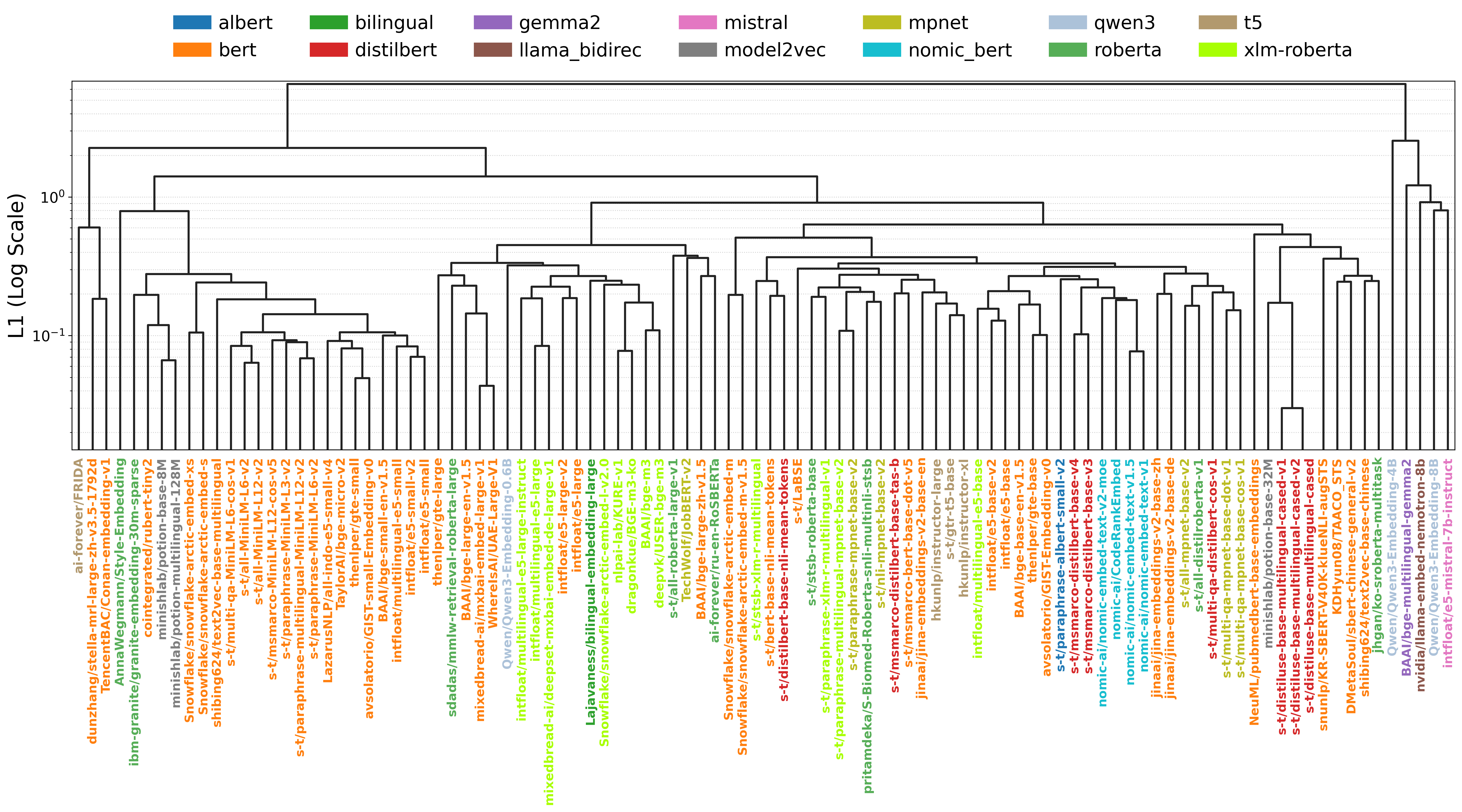}
    \caption{Hierarchical clustering of the top 100 downloaded encoders, coloured by encoder type. $\ell_1$ values are reported in log scale for better visualisation. A zoomed-in version shown in~\autoref{sec:zoomin_den}.}
    \label{fig:dendrogram_top_100}
\end{figure}

\begin{table}[t!]
\centering
\small
\begin{tabular}{l c c}
\toprule
\textbf{Task} & \textbf{Spearman} & \textbf{Pearson} \\
\midrule
BiorxivClusteringS2S & 0.876 & 0.837 \\
ArxivClusteringS2S & 0.874 & 0.855 \\
ArxivClusteringP2P & 0.862 & 0.833 \\
SciDocsRR & 0.861 & 0.820 \\
SciFact & 0.861 & 0.852 \\
HotpotQA & 0.860 & 0.816 \\
NFCorpus & 0.857 & 0.816 \\
RedditClusteringP2P & 0.855 & 0.838 \\
BiorxivClusteringP2P & 0.853 & 0.845 \\
CQADupstackPhysicsRetrieval & 0.848 & 0.797 \\
\bottomrule
\end{tabular}
\caption{Average Spearman/Pearson correlations between the true and predicted task performance of selected tasks, sorted by descending Spearman.}
\label{tab:individual_task_performance}
\end{table}

\begin{table}[t!]
\centering
\small
\begin{tabular}{l c c c}
\toprule
\textbf{Task Type} & \textbf{Spearman} & \textbf{Pearson} & \textbf{\# Datasets} \\
\midrule
Clustering & 0.827 & 0.826 & 11 \\
Reranking & 0.792 & 0.758 & 4 \\
Retrieval & 0.785 & 0.759 & 26 \\
PairClassification & 0.702 & 0.678 & 3 \\
STS & 0.538 & 0.459 & 10 \\
Classification & 0.473 & 0.438 & 12 \\
\bottomrule
\end{tabular}
\caption{Average Spearman/Pearson correlations between the true and predicted task performance by task type, sorted by descending Spearman correlation.}
\label{tab:task_type_spearman}
\end{table}

\begin{figure}[t!]
    \centering    
    \includegraphics[width=0.75\linewidth]{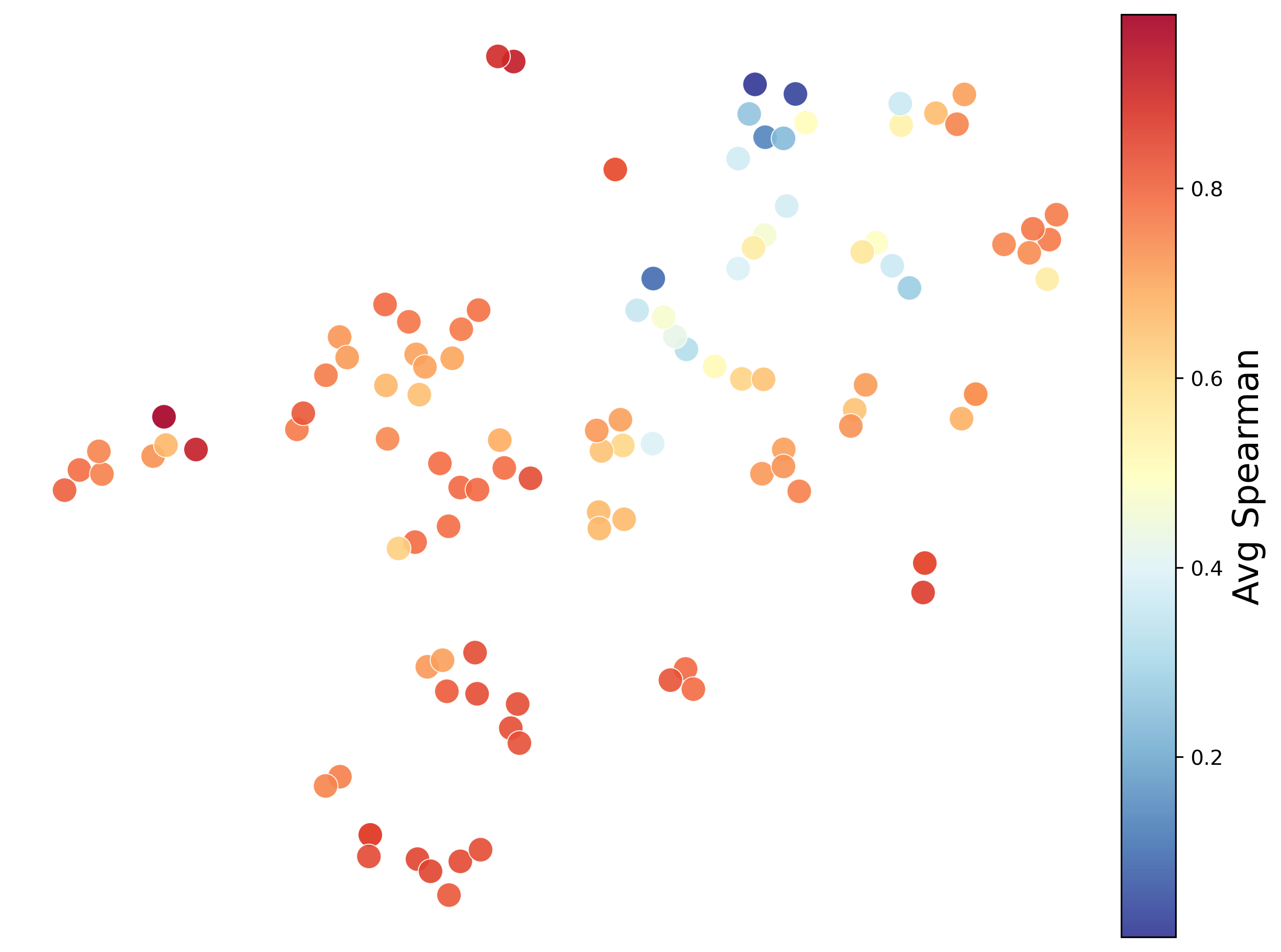}
    \caption{Submap of 112 encoders coloured by the average Spearman correlation between the true and predicted task performance of selected tasks.}
    \label{fig:submaps}
\end{figure}

\subsection{Correlation with Downstream Tasks}
\label{sec:exp:downstream}

A map that correlates well with the downstream task performance of the encoders is desirable because it shows that the map captures the true performance of encoders and can potentially be used to infer the performance of a novel encoder.
For this purpose, we measure the correlation between the \emph{true} performance of an encoder and its performance \emph{predicted} from the feature vector of that encoder.
We use official results from the MTEB leader-board as the true encoder performance instead of re-evaluating encoders by ourselves for reproducibility and computational feasibility.\footnote{Re-evaluating all encoders on MTEB requires substantial computational resources and many weeks.}
Out of all our encoders, 112 appear on the MTEB leader-board (see encoder list in~\autoref{sec:full_encoder_list_MTEB} and their results in~\autoref{sec:full_MTEB}), which provide a sufficiently large sample to test for statistical significance when predicting the task performance of encoders from their feature vectors.
We first use \href{https://scikit-learn.org/stable/modules/generated/sklearn.decomposition.PCA.html}{PCA} to project our feature vectors from 10,000 to a 50-dimensional space to mitigate any overfitting.
Next, we use \href{https://scikit-learn.org/stable/modules/generated/sklearn.linear_model.ElasticNetCV.html}{ElasticNetCV} with 5-fold cross-validation to predict the downstream task performance of the selected encoders using their (low-dimensional) feature vectors (see~\autoref{sec:impletation_details_mteb} for implementation details).

\autoref{tab:individual_task_performance} shows the performance for individual tasks.
Our encoder feature vectors report significantly high Spearman and Pearson correlations ($>$0.8) ($p\text{-value}<$0.05) in multiple tasks, indicating their relatedness to downstream task performance.
We further create submaps for the 112 encoders by their average performance on the selected 10 tasks in~\autoref{tab:individual_task_performance}.\footnote{We use min-max normalisation for each task's performance and then compute their average.}
As shown in \autoref{fig:submaps}, our map reflects the average Spearman correlation between the true and predicted MTEB performance for encoders in general, which is demonstrated by the proximity of encoders with similar levels of average Spearman correlation.
For example, encoders with higher Spearman correlation (red/orange) are grouped at the bottom-left and middle-left areas of the map, while encoders with lower Spearman correlation (yellow/blue) are grouped at the upper-middle area.

Feature vectors are especially connected well with task types such as reranking, clustering and retrieval, as shown in \autoref{tab:task_type_spearman}.
In contrast, classification tasks show weaker correlations.
We believe this is because they require training a classification heads using class labelled training instances, which are not available to the encoders, hence not captured in the encoder feature vectors.
STS tasks report a moderate Spearman Correlation.
Notably, our feature vectors have a high Spearman correlation $\geq$0.75 on the STS16 and STS13 task.

\section{Conclusion}
We propose a \ac{QRE}-based method to represent diverse sentence encoders as feature vectors in a shared space, creating a map of sentence encoders.
Our map has well-founded properties and accurately reflects relationships between encoders, which are verified both theoretically and empirically.
Moreover, our map correlates well with downstream task performance of encoders.

\section{Limitations}
\label{sec:limitations}

Our \ac{QRE}-based method to create feature vectors for encoders and visualise them in a shared two-dimensional space can be applied to any sentence encoder and any sentence set, which is not specific to English.
However, our selected sentence set to create embeddings for encoders as representations of the embedding spaces only covers English, which is a morphologically limited language.
Our map is applicable to a multilingual sentence set, while the quality of the maps is based on the ability of the sentence encoders to accurately capture the semantic meaning of the multilingual sentences, which is out of our scope.

We select a set of 10,000 sentences.
A larger sentence set (e.g. 100,000 sentences) can be selected to represent the semantic spaces of sentence encoders.
However, in~\autoref{sec:select_10000}, our experimental results have shown that 10,000 sentences perform well enough to create high-quality maps.
Additionally, our \ac{QRE} method is scalable to larger sentence sets, because our method does not need to compute \ac{PIP} matrices explicitly, which instead uses \ac{SVD} of the embedding matrices as detailed in~\autoref{sec:method}.

In the selected set of sentence encoders, we contain the multilingual sentence encoders and map them successfully based on the English set.
The effects of non-English sentence sets for such multilingual encoders in our map remain to be further investigated in future work.

At the time of writing, there are 17,515 sentence encoders in the Hugging Face Hub, we only select 2000 of them based on the most downloaded records, due to computational resource constraints.
Among the 2000 selected encoders, we further filter out those encoders that cannot be loaded or not correctly used by the \href{https://huggingface.co/sentence-transformers}{SentenceTransformer} class such as cross encoders, as explained in~\autoref{sec:encoder_selection}.
Cross encoders require sentence pairs (query, document) as input, while we embed one sentence at a time.
Therefore, extending our method to heterogeneous encoders of diverse architectures could further reveal the relationships between sentence encoders.

Our selection strategy for sentence encoders in~\autoref{sec:encoder_selection} is based on the number of downloads, introducing a popularity bias.
This strategy tends to include fewer encoders that are trained on low-resource languages.
Therefore, it would be important to particularly consider low-resource encoders, facilitating a fair and linguistically diverse analysis.

\section{Ethical Considerations}
We do not annotate or release any datasets in this project.
To the best of our knowledge, there are no ethical issues raised regarding the high-quality-english-sentences dataset (\autoref{sec:sec:sentence_set}) we used.
However, it has been reported that unfair social biases are found in some MTEB tasks such as STS~\cite{webster2021measuring}.
Although we do not train any sentence encoders, it has been reported that both MLM-based and LLM-based sentence encoders contain various social biases~\cite{may2019measuring, lin2025investigating}.
We have not evaluated the social biases of the sentence encoders and their downstream implications.
Therefore, we consider it to be important to measure social biases before integrating any encoder into the map.

\bibliography{myrefs.bib}

\appendix

\section{Eigenvalue Properties of the PIP Matrix}
\label{sec:PIP_properties}

\begin{theorem}
Let $\mat{A} \in \mathbb{R}^{N \times d}$ be a real matrix with $N \geq d$. Let $\sigma_1, \sigma_2, \dots, \sigma_d$ denote the eigenvalues of $\mat{A}$. Let $\mat{B} = \mat{A}\mat{A}\T$. Let $\lambda_1, \lambda_2, \dots, \lambda_N$ denote the eigenvalues of $\mat{B}$. The following properties hold:
\begin{enumerate}
    \item[(a)] All eigenvalues $\lambda_i$ are real and non-negative ($\lambda_i \geq 0$).
    \item[(b)] The first $d$ eigenvalues are equal to the squared singular values of $A$ (i.e. $\lambda_i = \sigma_i^2$ for $1 \leq i \leq d$).
    \item[(c)] The remaining $(N-d)$ eigenvalues are zero.
    \item[(d)] For any $c>0$, the matrix $\mat{C} = \frac{1}{c}\mat{B}$ remains \ac{PSD}.
\end{enumerate}
\end{theorem}

\begin{proof}
\textbf{Part (a): Real and Non-negative}

First, we establish that $\mat{B}$ is symmetric. Using the property $(\mat{X}\mat{Y})\T = \mat{Y}\T\mat{X}\T$:
\[
\mat{B}\T = (\mat{A}\mat{A}\T)\T = (\mat{A}\T)\T \mat{A}\T = \mat{A}\mat{A}\T = \mat{B}.
\]
Since $\mat{B}$ is real and symmetric, by the Spectral Theorem, all its eigenvalues are real.

Next, we show that $\mat{B}$ is \ac{PSD}. Let $\lambda$ be an eigenvalue of $\mat{B}$ and $\vec{v} \in \R^N$ be the corresponding non-zero eigenvector.
\begin{align*}
    \mat{B}\vec{v} &= \lambda \vec{v} \\
    \vec{v}\T \mat{B} \vec{v} &= \vec{v}\T (\lambda \vec{v}) \\
    \vec{v}\T \mat{A} \mat{A}\T \vec{v} &= \lambda (\vec{v}\T \vec{v}) \\
    (\mat{A}\T\vec{v})\T (\mat{A}\T\vec{v}) &= \lambda \norm{\vec{v}}^2 \\
    \norm{\mat{A}\T\vec{v}}^2 &= \lambda \norm{\vec{v}}^2.
\end{align*}
Since the squared Euclidean norm is non-negative, $\norm{\mat{A}\T\vec{v}}^2 \geq 0$ and $\norm{\vec{v}}^2 > 0$, it follows that:
\[
\lambda = \frac{\norm{\mat{A}\T\vec{v}}^2}{\norm{\vec{v}}^2} \geq 0.
\]

\vspace{1em}
\noindent \textbf{Part (b) and (c): Magnitude of Eigenvalues}

We utilize the Singular Value Decomposition (SVD) of $\mat{A}$. Let $\mat{A} = \mat{U} \mat{\Sigma} \mat{V}\T$, where:
\begin{itemize}
    \item $\mat{U} \in \R^{N \times N}$ is an orthogonal matrix.
    \item $\mat{V} \in \R^{d \times d}$ is an orthogonal matrix.
    \item $\mat{\Sigma} \in \R^{N \times d}$ is a rectangular diagonal matrix containing singular values $\sigma_i$.
\end{itemize}

Substituting the SVD into the expression for $\mat{B}$:
\begin{align*}
    \mat{B} &= (\mat{U} \mat{\Sigma} \mat{V}\T) (\mat{U} \mat{\Sigma} \mat{V}\T)\T \\
    &= \mat{U} \mat{\Sigma} \mat{V}\T \mat{V} \mat{\Sigma}\T \mat{U}\T.
\end{align*}
Since $\mat{V}$ is orthogonal, $\mat{V}\T\mat{V} = \mat{I}_d$. Thus:
\[
    \mat{B} = \mat{U} (\mat{\Sigma} \mat{\Sigma}\T) \mat{U}\T.
\]
This equation represents the eigendecomposition of $\mat{B}$, where the columns of $\mat{U}$ are the eigenvectors and the diagonal entries of the matrix $\mat{\Lambda} = \mat{\Sigma} \mat{\Sigma}\T$ are the eigenvalues.

We calculate $\mat{\Sigma} \mat{\Sigma}\T$ (an $N \times N$ matrix):
\begin{align}
    \mat{\Sigma} \mat{\Sigma}\T = \diag(\sigma_1, \dots, \sigma_d, 0, \dots, 0) \times \nonumber \\ 
    \diag(\sigma_1, \dots, \sigma_d, 0, \dots, 0)\T \nonumber
\end{align}
The resulting diagonal entries are:
\[
(\mat{\Sigma} \mat{\Sigma}\T)_{ii} = 
\begin{cases} 
\sigma_i^2 & \text{for } 1 \leq i \leq d \\
0 & \text{for } d < i \leq N 
\end{cases}
\]
Thus, the first $d$ eigenvalues are $\sigma_i^2$, and the remaining $N-d$ eigenvalues are $0$.

\textbf{Part (d): Positive Semi-Definiteness Preservation}

As shown in (a), $\mat{B}$ is \ac{PSD}.
By definition of \ac{PSD} matrices, we have for any vector $x \in \mathbb{R}^Ns$ :
\begin{align}
    x\T \mat{B} x \geq 0 \nonumber
\end{align}
Given $\mat{B}$ is symmetric, we first show the symmetry of $\mat{C}$
\begin{align}
    \mat{C}\T = \left(\frac{1}{c} \mat{B}\right)\T=\frac{1}{c} \mat{B}\T=\frac{1}{c} \mat{B}=\mat{C} \nonumber
\end{align}
Then, we show that for any vector $x$:
\begin{align}
    x\T \mat{C} x=x\T\left(\frac{1}{c} \mat{B}\right) x=\frac{1}{c}\left(x\T \mat{B} x\right)\geq 0 \nonumber
\end{align}
because $x\T \mat{B} x \geq 0$ and $\frac{1}{c}>0$.
Therefore, $\mat{C}$ is \ac{PSD}.

\end{proof}

\section{Proof for Eigen Decomposition of QRE}
\label{sec:QRE-eigen}

Here we prove the derivation from \eqref{eq:QRE} to \eqref{eq:QRE_eigen}.

\begin{proof}
\textbf{Derivation of Term 1:} $\Tr(\mat{\rho} \ln \mat{\rho})$

First, we compute the matrix logarithm of $\mat{\rho}$. For a diagonalized matrix, we apply the logarithm to the eigenvalues:
\begin{align}
    \ln \mat{\rho} = \sum_k \ln (\lambda_k) \vec{v}_k \vec{v}_k\T
\end{align}
Then, multiply $\mat{\rho}$ by $\ln \mat{\rho}$:
\begin{align}
    \mat{\rho} (\ln \mat{\rho}) = \left(\sum_i \lambda_i \vec{v}_i \vec{v}_i\T\right) \left(\sum_k \ln (\lambda_k) \vec{v}_k \vec{v}_k\T\right)
\end{align}

We expand the product. Since scalars commute, we can rearrange the terms. 
\begin{align}
    = \sum_i \sum_k \lambda_i \ln (\lambda_k) \vec{v}_i (\vec{v}_i\T \vec{v}_k) \vec{v}_k\T
\end{align}

Using the orthonormality, $\vec{v}_i\T \vec{v}_k = \delta_{ik}$, the terms are zero unless $i=k$:
\begin{align}
    = \sum_i \lambda_i \ln (\lambda_i) \vec{v}_i \vec{v}_i\T
\end{align}

Then, we take the matrix Trace on both sides. 
Note that the Trace of a matrix is linear, and we can take the summation out from the trace:
\begin{align}
    \Tr(\mat{\rho} \ln \mat{\rho}) = \sum_i \lambda_i \ln (\lambda_i) \Tr(\vec{v}_i \vec{v}_i\T)
\end{align}

Using the cyclic property of the Trace (i.e. $\Tr(\mat{A}\mat{B}) = \Tr(\mat{B}\mat{A})$), we have
\begin{align}
    \Tr(\vec{v}_i \vec{v}_i\T) &= \Tr(\vec{v}_i\T \vec{v}_i) \\ \nonumber
    &= \vec{v}_i\T \vec{v}_i = \norm{\vec{v}_i}^2 = 1
\end{align}
Thus,
\begin{align}
    \Tr(\mat{\rho} \ln \mat{\rho}) = \sum_i \lambda_i \ln \lambda_i
\end{align}

\vspace{1em}
\textbf{Derivation of Term 2:} $\Tr(\mat{\rho} \ln \mat{\sigma})$

First, write out $\ln \mat{\sigma}$:
\begin{align}
    \ln \mat{\sigma} = \sum_j \ln (\mu_j) \vec{u}_j \vec{u}_j\T
\end{align}

Substitute the expansions of $\mat{\rho}$ and $\ln \mat{\sigma}$ into the Trace:
\par\nobreak
{\small
\vspace{-2mm}
\begin{align}
    \Tr(\mat{\rho} \ln \mat{\sigma}) = \Tr\left( \left[\sum_i \lambda_i \vec{v}_i \vec{v}_i\T\right] \left[\sum_j \ln (\mu_j) \vec{u}_j \vec{u}_j\T\right] \right)
\end{align}
}

Pull the summations and scalars ($\lambda_i$ and $\ln \mu_j$) outside the Trace:
\begin{align}
    = \sum_i \sum_j \lambda_i \ln (\mu_j) \cdot \Tr(\vec{v}_i \vec{v}_i\T \vec{u}_j \vec{u}_j\T)
\end{align}

Focus on the term inside the Trace: $\vec{v}_i \vec{v}_i\T \vec{u}_j \vec{u}_j\T$. Matrix multiplication is associative. We can group the middle terms $(\vec{v}_i\T \vec{u}_j)$. Note that $\vec{v}_i\T \vec{u}_j$ is the dot product of two vectors, resulting in a scalar value.
\begin{align}
    \vec{v}_i (\vec{v}_i\T \vec{u}_j) \vec{u}_j\T = (\vec{v}_i\T \vec{u}_j) \cdot (\vec{v}_i \vec{u}_j\T)
\end{align}

Substitute this back into the Trace:
\begin{align}
    \Tr\left( (\vec{v}_i\T \vec{u}_j) \cdot \vec{v}_i \vec{u}_j\T \right) = (\vec{v}_i\T \vec{u}_j) \cdot \Tr(\vec{v}_i \vec{u}_j\T)
\end{align}

By applying the cyclic property of the Trace again to $\Tr(\vec{v}_i \vec{u}_j\T)$:
\begin{align}
    \Tr(\vec{v}_i \vec{u}_j\T) = \Tr(\vec{u}_j\T \vec{v}_i)
\end{align}

Note that for real vectors, the dot product is symmetric, so $\vec{u}_j\T \vec{v}_i = \vec{v}_i\T \vec{u}_j$.
\begin{align}
    = \vec{v}_i\T \vec{u}_j
\end{align}

Combine the parts:
\begin{align}
    \text{Trace Term} = (\vec{v}_i\T \vec{u}_j) \cdot (\vec{v}_i\T \vec{u}_j) = (\vec{v}_i\T \vec{u}_j)^2
\end{align}

Therefore, Term 2 becomes:
\begin{align}
    \Tr(\mat{\rho} \ln \mat{\sigma}) = \sum_i \lambda_i \sum_j (\vec{v}_i\T \vec{u}_j)^2 \ln \mu_j
\end{align}

Subtract Term 2 from Term 1:
\par\nobreak
{\small
\vspace{-2mm}
\begin{align}
    S(\mat{\rho} \Vert \mat{\sigma}) = \left(\sum_i \lambda_i \ln \lambda_i\right) - \left(\sum_i \lambda_i \sum_j (\vec{v}_i\T \vec{u}_j)^2 \ln \mu_j\right)
\end{align}
}

Rearranging the parenthesis to match the target format:
\par\nobreak
{\small
\vspace{-2mm}
\begin{align}
    S(\mat{\rho} \Vert \mat{\sigma}) = \sum_i \lambda_i (\ln \lambda_i) - \sum_i \lambda_i \left(\sum_j (\vec{v}_i\T \vec{u}_j)^2 \ln \mu_j\right)
\end{align}
}
\end{proof}

\section{Proof for QRE Estimation with Residual Mass}
\label{sec:QRE-mass}

\begin{utheorem}
Let $\mat{\rho}$ be the density matrix of a base encoder with non-zero orthonormal eigenpairs $\{(\lambda_i, \vec{v}_i)\}_{i=1}^{K_{\rho}}$. 
Let $\mat{\sigma}$ be the density matrix of a target encoder with non-zero orthonormal eigenpairs $\{(\mu_j, \vec{u}_j)\}_{j=1}^{K_{\sigma}}$. 
By perturbating $\mat{\sigma}$ with a small noise parameter $\epsilon > 0$ on its null space, the \ac{QRE} can be approximated as follows:
\begin{align}
    S(\mat{\rho} \Vert \mat{\sigma}_\epsilon) = \sum_{i=1}^{K_{\rho}} \lambda_i (\ln \lambda_i) - \sum_{i=1}^{K_{\rho}} \lambda_i (\mathcal{C}_i  + r_i \ln \epsilon)
\end{align}.
where $c_i = \sum_{j=1}^{K_{\sigma}} (\vec{v}_i\T \vec{u}_j)^2$ is the captured mass, $r_i = 1 - c_i$ is the residual mass, and $\mathcal{C}_i = \sum_{j=1}^{K_{\sigma}} (\vec{v}_i\T \vec{u}_j)^2 \ln \mu_j$ is the cross-entropy contribution from the aligned subspace.
\end{utheorem}

\begin{proof}

We construct a regularised target matrix $\mat{\sigma}_\epsilon$ by adding a small perturbation $\epsilon$ to the zero eigenvalues corresponding to the null space of $\mat{\sigma}$. 
Let the subspace spanned by the target encoder be $\mathcal{U}=\operatorname{span}\{\vec{u}_1, \ldots, \vec{u}_{K_{\sigma}}\}$. 
Let $\{\vec{u}_{{K_{\sigma}}+1}, \ldots, \vec{u}_N\}$ be an orthonormal basis for the orthogonal complement $\mathcal{U}^{\perp}$ (the null space). 
We define $\mat{\sigma}_\epsilon$ as:
\begin{align}
    \label{eq:pad_sigma}
     \mat{\sigma}_\epsilon = \mat{\sigma} + \epsilon \sum_{k={K_{\sigma}}+1}^N \vec{u}_k \vec{u}_k\T
\end{align}

Without loss of generality, assume $K_{\sigma} \leq K_{\rho}$.\footnote{With unit base encoder used as $\mat{\rho}$, we always have $K_{\sigma} \leq K_{\rho}$.}
The spectral decomposition of $\mat{\sigma}_\epsilon$ is therefore the sum of the original active components and the noise components:
\begin{align}
     \mat{\sigma}_\epsilon = \sum_{j=1}^{K_{\sigma}} \mu_j \vec{u}_j \vec{u}_j\T + \sum_{k={K_{\sigma}}+1}^{K_{\rho}} \epsilon \vec{u}_k \vec{u}_k\T
\end{align}

Note that the set of all eigenvectors $\{\vec{u}_j\}_{j=1}^{K_{\sigma}} \cup \{\vec{u}_k\}_{k={K_{\sigma}}+1}^{K_{\rho}}$ now forms a complete orthonormal basis for the effective subspace $\R^{K_{\rho}}$.

Recalling \eqref{eq:QRE}, the definition of \ac{QRE} is:
\begin{align}
     S(\mat{\rho} \Vert \mat{\sigma}_\epsilon) = \Tr(\mat{\rho} \ln \mat{\rho}) - \Tr (\mat{\rho} \ln \mat{\sigma}_\epsilon)
\end{align}

The first term is simply the negative Von Neumann entropy of $\mat{\rho}$:
\begin{align}
     \Tr(\mat{\rho} \ln \mat{\rho}) = \sum_i \lambda_i \ln \lambda_i
\end{align}

We focus on the second term (Cross Entropy), expanded using the eigenvectors of $\mat{\rho}$ (denoted $\{\vec{v}_i\}$):
\begin{align}
     \Tr(\mat{\rho} \ln \mat{\sigma}_\epsilon) = \sum_i \lambda_i \vec{v}_i\T (\ln \mat{\sigma}_\epsilon) \vec{v}_i
\end{align}

As we already clarified in \autoref{sec:density_matrix}, $\ln$ is the matrix logarithm, not the element-wise logarithm.
For a symmetric \ac{PSD} matrix $\mat{\rho}$ with its spectral decomposition $\mat{\rho} = \sum \lambda_i \vec{v}_i \vec{v}_i\T$, the matrix logarithm is defined by the logarithm of its eigenvalues: $\ln \mat{\rho} = \sum (\ln \lambda_i) \vec{v}_i \vec{v}_i\T$.

Using the spectral decomposition of $\mat{\sigma}_\epsilon$, the matrix logarithm $\ln \mat{\sigma}_\epsilon$ is:
\begin{align}
     \ln \mat{\sigma}_\epsilon = \sum_{j=1}^{K_{\sigma}} (\ln \mu_j) \vec{u}_j \vec{u}_j\T + \sum_{k={K_{\sigma}}+1}^{K_{\rho}} (\ln \epsilon) \vec{u}_k \vec{u}_k\T
\end{align}

Substituting this into the quadratic form $\vec{v}_i\T (\dots) \vec{v}_i$:
\begin{align}
\vec{v}_i\T (\ln \mat{\sigma}_\epsilon) \vec{v}_i &= \underbrace{\sum_{j=1}^{K_{\sigma}} (\ln \mu_j) \vec{v}_i\T (\vec{u}_j \vec{u}_j\T) \vec{v}_i}_{\text{Aligned Term}} \nonumber \\
&+ \underbrace{\sum_{k={K_{\sigma}}+1}^{K_{\rho}} (\ln \epsilon) \vec{v}_i\T (\vec{u}_k \vec{u}_k\T) \vec{v}_i}_{\text{Orthogonal Term}}
\end{align}

Given that $\{\vec{v}_i\}$ and $\{\vec{u}_j\}$ are orthonormal bases, we have the properties that $\norm{\vec{v}_i}^2 =\norm{\vec{u}_j}^2 = 1$, and they satisfy the orthogonality condition:
\[
\vec{v}_i\T \vec{v}_k = \delta_{ik} \quad \text{and} \quad \vec{u}_j\T \vec{u}_l = \delta_{jl}
\]
where $\delta_{ik}$ is the Kronecker delta, which equals 1 if $i=k$ and 0 otherwise.

Recognizing that 
\begin{align}
\vec{v}_i\T \vec{u}_j \vec{u}_j\T \vec{v}_i = (\vec{v}_i\T \vec{u}_j) (\vec{u}_j\T \vec{v}_i) = (\vec{v}_i\T \vec{u}_j)^2
\end{align}
Then,
\begin{align}
\label{eq:expansion}
\vec{v}_i\T (\ln \mat{\sigma}_\epsilon) \vec{v}_i & = \sum_{j=1}^{K_{\sigma}} (\ln \mu_j) (\vec{v}_i\T \vec{u}_j)^2 + \\ \nonumber 
& \sum_{k={K_{\sigma}}+1}^{K_{\rho}} (\ln \epsilon) (\vec{v}_i\T \vec{u}_k)^2
\end{align}

% The term $\sum_{j=1}^{K_{\sigma}} (\vec{v}_i\T \vec{u}_j)^2$ is exactly the captured mass $m_i$.

Since the basis $\{\vec{u}_k\}_{k=1}^{K_{\rho}}$ forms a complete orthonormal basis for $\R^{K_{\rho}}$,\footnote{This is the case for our proposed unit base encoder.} \textbf{Parseval's identity} asserts that the squared norm of any unit vector $\vec{v}_i$ equals the sum of its squared projection coefficients:
\begin{align}
    \label{eq:parseval}
    \norm{\vec{v}_i}^2 = \sum_{k=1}^{K_{\rho}} (\vec{v}_i\T \vec{u}_k)^2 = 1
\end{align}

We decompose this sum into components lying in the aligned subspace $\mathcal{U}$ (spanned by $\{\vec{u}_j\}_{j=1}^{K_{\sigma}}$) and the orthogonal null space $\mathcal{U}^\perp$ (spanned by $\{\vec{u}_k\}_{k=K_{\sigma}+1}^{K_{\rho}}$). 
This decomposition corresponds to the \textbf{generalized Pythagorean theorem}, which states that $\norm{\vec{v}_i}^2 = \norm{\mathcal{P}_{\mathcal{U}}(\vec{v}_i)}^2 + \norm{\mathcal{P}_{\mathcal{U}^\perp}(\vec{v}_i)}^2$, where $\mathcal{P}$ denotes the projection operator.
Substituting the projection coefficients into the theorem yields:
\begin{align}
    \label{eq:pythagoras}
    \underbrace{\sum_{j=1}^{K_{\sigma}} (\vec{v}_i\T \vec{u}_j)^2}_{\text{Squared projection on } \mathcal{U}} + \underbrace{\sum_{k={K_{\sigma}}+1}^{K_{\rho}} (\vec{v}_i\T \vec{u}_k)^2}_{\text{Squared projection on } \mathcal{U}^\perp} = 1
\end{align}

Recognizing that the first term is the captured mass $c_i$, we rearrange the equation to solve for the projection onto the null space:
\par\nobreak
{\small
\vspace{-4mm}
\begin{align}
\sum_{k={K_{\sigma}}+1}^{K_{\rho}} (\vec{v}_i\T \vec{u}_k)^2 = 1 - \sum_{j=1}^{K_{\sigma}} (\vec{v}_i\T \vec{u}_j)^2 = 1 - c_i = r_i
\end{align}
}

Substituting $m_i$ and $r_i$ back into the expansion in \eqref{eq:expansion}:
\begin{align}
\vec{v}_i\T (\ln \mat{\sigma}_\epsilon) \vec{v}_i = \left( \sum_{j=1}^{K_{\sigma}} (\vec{v}_i\T \vec{u}_j)^2 \ln \mu_j \right) + r_i \ln \epsilon
\end{align}

Combining the terms, we arrive at the computational form used in \eqref{eq:qre-mass-proof}. 
We define the cross-entropy contribution from the aligned subspace as $\mathcal{C}_i = \sum_{j=1}^{K_{\sigma}} (\vec{v}_i\T \vec{u}_j)^2 \ln \mu_j$.
\begin{align}
S(\mat{\rho} \Vert \mat{\sigma}_\epsilon) & = \sum_{i=1}^{K_{\rho}} \lambda_i (\ln \lambda_i) - \sum_{i=1}^{K_{\rho}} \lambda_i (\mathcal{C}_i + r_i \ln \epsilon)
\end{align}

\end{proof}

\section{Theoretical Analysis of the QRE-based Encoder Feature Space}
\label{sec:feat-space}

In \autoref{sec:encoder_embedding}, we created a feature vector $\vec{\phi}_m$ to represent an encoder $f_m$ using the \ac{QRE} $S(\mat{\rho}_{0}||\mat{\rho}_m)$ of $\mat{\rho}_m$ (i.e. the density matrix corresponding to $f_m$) with respect to $\mat{\rho_{0}}$ (i.e. the density matrix corresponding to the unit base encoder, $f_0$).
Given that our goal is to create a map where each target encoder is represented with respect to its relationship to the other encoders, we consider it to be insightful to mathematically analyse the encoder feature vector space further.

For this purpose, let us consider two target encoders $f_m$ and $f_m'$ with their respective density matrices $\mat{\rho}_m$ and $\mat{\rho}_m'$.
For simplicity of the disposition, we assume both  $\mat{\rho}_m$ and $\mat{\rho}_m'$ to be full rank in $\R^N$.
Therefore, we can denote the eigenpairs for $\mat{\rho}_m$ and $\mat{\rho}_m'$ respectively as
$\{(\mu_j, \vec{u}_j)\}_{j=1}^{N}$ and $\{(\mu'_k, \vec{u}'_k)\}_{k=1}^{N}$.
For rank deficient cases we can add a small random perturbation to the eigenvalues to make the density matrices full rank.

Following \eqref{eq:axis_ele}, we can write the $w$-th dimensions of $\vec{\phi}_m$ and  $\vec{\phi}_m'$ as follows.
\par\nobreak
{\small
\vspace{-3mm}
\begin{align}
\label{eq:m-w}
    \phi_{m,w} &= \lambda_w \ln \lambda_w - \lambda_w \sum_{j=1}^{N} \left( \vec{v}_w\T\vec{u}_{j}\right)^2 \ln \mu_j \\ 
\label{eq:m-prime-w}
    \phi_{m',w} &= \lambda_w \ln \lambda_w - \lambda_w \sum_{k=1}^{N} \left( \vec{v}_w\T\vec{u}'_{k}\right)^2 \ln \mu'_k
\end{align}
}%

Because the unit base encoder has all of its eigenvalues $\lambda_w = 1/N$ and eigenvectors $\vec{v}_w$ as the $w$-th unit vector, we can substitute those values in \eqref{eq:m-w} and \eqref{eq:m-prime-w} to further simplify as follows.
\par\nobreak
{\small
\vspace{-3mm}
\begin{align}
\label{eq:m-w}
    \phi_{m,w} &= -\frac{1}{N}\ln N- \frac{1}{N} \sum_{j=1}^{N} \left( \vec{v}_w\T\vec{u}_{j}\right)^2 \ln \mu_j \\ 
\label{eq:m-prime-w}
    \phi_{m',w} &=  -\frac{1}{N}\ln N - \frac{1}{N} \sum_{k=1}^{N} \left( \vec{v}_w\T\vec{u}'_{k}\right)^2 \ln \mu'_k
\end{align}
}%

In particular, we are interested in the \emph{offset}, $(\phi_{m,w} -  \phi_{m',w})$, between the $w$-th feature of $f_m$ and $f_m'$, which is given by \eqref{eq:offset}.
\par\nobreak
{\small
\vspace{-3mm}
\begin{align}
    \label{eq:offset}
      -\frac{1}{N} \sum_{j=1}^{N} \left( \vec{v}_w\T\vec{u}_{j}\right)^2 \ln \mu_j + \frac{1}{N} \sum_{k=1}^{N} \left( \vec{v}_w\T\vec{u}'_{k}\right)^2 \ln \mu'_k
\end{align}
}
We see that the self-entropy terms for the unit base encoder cancel out in this offset.

Moreover, recall that the inner-product between the $w$-th unit vector $\vec{v}_w$ and $\vec{u}_j$ is simply selecting the $w$-th dimension $u_{j,w}$ of  $\vec{u}_j$.
Therefore, cross-entropy terms can be further simplified as follows:
\par\nobreak
{\small
\vspace{-3mm}
\begin{align}
    \label{eq:cross-simple}
    -\frac{1}{N} \sum_{j}^{N}\left(\vec{v}_w\T\vec{u}_j\right)^2\ln \mu_j &=  -\frac{1}{N} \sum_{j}^{N} \left(u_{j,w}\right)^2 \ln \mu_j  \\
    &= \frac{1}{N}  \sum_{j}^{N} \left(u_{j,w}\right)^2 (-\ln \mu_j)   \\
    &= \frac{1}{N}  \sum_{j}^{N} \left(u_{j,w} \sqrt{(-\ln \mu_j)} \right)^2 \label{eq:neg} \\
    &= \frac{1}{N} \norm{\bar{\vec{u}_w}}^2 . \label{eq:bar}
\end{align}
}%
In \eqref{eq:neg} we used the fact that all eigenvalues of density matrices are in $[0,1]$, hence $(-\ln \mu_j) > 0$ (in practice $\mu_j \in (0,1)$).
% Moreover, we define $\bar{\vec{u}}$ as version of $\vec{u}$ where its $j$-th dimension is multiplied by $(-\ln \mu_j)$.
Moreover, we define $\bar{\vec{u}}_w$ as a vector in $\mathbb{R}^N$ whose $j$-th component is $ u_{j,w} \sqrt{-\ln \mu_j}$.

Plugging these results back, we can evaluate the offset as given by \eqref{eq:offset-simplified}.
\begin{align}
    \label{eq:offset-simplified}
    \phi_{m,w} -  \phi_{m',w} = \frac{1}{N} \left( \norm{\bar{\vec{u}_w}}^2 - \norm{\bar{\vec{u}'_w}}^2\right)
\end{align}

From \eqref{eq:offset-simplified}, we see that the offset is independent of the density matrix of the unit base encoder and depends only on the weighted projection of the eigenvectors of the target encoders onto the $w$-th basis vector of the unit base encoder.

\section{Synthetic Example}
\label{sec:illu}

To show the faithfulness of our method, we provide an example of synthetic embeddings.
We create synthetic embeddings with different levels of noise added to the 10,000-dimensional identity density matrix of the unit base encoder.
For the embedding matrix of the base encoder, we add noise perturbation to it as follows.

For each embedding vector $x$ in the base encoder matrix, we generate a noise vector $\vec{n}$ from a multivariate Gaussian distribution:

\begin{align}
\vec{n} \sim \mathcal{N}\left(0, \sigma^2 \mathbf{I}_d\right)
\end{align}

Recall that embedding matrices are normalised.
To ensure the noise is scaled and does not dominate the original signal, we calculate the perturbed synthetic embedding $\hat{x}$ as:

\begin{align}
\hat{\vec{x}}=\vec{x}+0.5 \cdot \frac{\vec{n}}{\norm{\vec{n}}}
\end{align}

We set two noise levels with intervals $[0,1]$ and $[3,4]$ (low noise and high noise), and uniformly at random select 10 values for $\sigma^2$ from the two intervals respectively.

For the synthetic embeddings, we compute the feature vectors using the method defined in \autoref{sec:encoder_embedding} and use t-SNE to visualise as described in \autoref{sec:t-sne}.
\autoref{fig:syn_map} shows that the synthetic representations with low noise are close to each other and are spatially located from the group with high noise.
Additionally, the low noise group has low \ac{QRE} values based on the base encoder, while the high noise group has high \ac{QRE} values.
These two results validate that our map can distinguish the similar encoders and accurately visualise them in the same space.

\section{Full Nearest Neighbours Table}
\label{sec:NNs_full_table}

We randomly select the 16 encoders from our encoder list detailed in \autoref{sec:encoder_selection}.
For each selected encoder, taken as the target encoder, we retrieve its top 10 nearest neighbours based on the lowest pairwise $\ell_1$ value relative to that encoder.
\autoref{tab:nearest_neighbors_16_10} shows the full table of the nearest neighbours.

\begin{table*}[t!]
\centering
\scriptsize
\renewcommand{\arraystretch}{0.8} 
\setlength{\tabcolsep}{2pt}
\resizebox{\textwidth}{!}{%
\begin{tabular}{l c @{\hspace{0.8cm}} l c}
\toprule
\textbf{Encoder Name} & \textbf{$\ell_1$} & \textbf{Encoder Name} & \textbf{$\ell_1$} \\
\midrule
% --- Block 1 ---
\multicolumn{2}{l}{\textbf{Target: multilingual-e5-base}} & \multicolumn{2}{l}{\textbf{Target: bge-small-en-v1.5}} \\
\quad Hiveurban/multilingual-e5-base & 0.0 & \quad optimum-intel-internal-testing/bge-small-en-v1.5 & 0.0 \\
\quad embaas/s-t-multilingual-e5-base & 0.0 & \quad michaelfeil/bge-small-en-v1.5 & 0.0 \\
\quad d0rj/e5-base-en-ru & 0.0012 & \quad JALLAJ/5epo & 0.0687 \\
\quad Renan1997/sentence-transformer & 0.0263 & \quad raul-delarosa99/bge-small-en-v1.5-RIRAG\_ObliQA & 0.0692 \\
\quad antoinelouis/french-me5-base & 0.0515 & \quad sdadas/mmlw-e5-small & 0.0710 \\
\quad clips/e5-base-trm-nl & 0.0763 & \quad austinpatrickm/finetuned\_bge\_embeddings\_v5\_small\_v1.5 & 0.0771 \\
\quad djovak/embedic-base & 0.1134 & \quad avsolatorio/GIST-small-Embedding-v0 & 0.0803 \\
\quad Lajavaness/bilingual-embedding-base & 0.1431 & \quad baconnier/Finance2\_embedding\_small\_en-V1.5 & 0.0818 \\
\quad KarBik/legal-french-matroshka & 0.1439 & \quad jebish7/MedEmbed-small-v0.1\_MNR\_1 & 0.0833 \\
\quad intfloat/e5-base-v2 & 0.1464 & \quad OrcaDB/gte-small & 0.0860 \\
\midrule
% --- Block 2 ---
\multicolumn{2}{l}{\textbf{Target: paraphrase-MiniLM-L6-v2}} & \multicolumn{2}{l}{\textbf{Target: pubmedbert-base-embeddings}} \\
\quad harsh-wk/long-seq-paraphrase-MiniLM-L6-v2 & $3.82 e^{-4}$ & \quad NeuML/pubmedbert-base-embeddings-matryoshka & 0.1398 \\
\quad s-t/paraphrase-MiniLM-L12-v2 & 0.0632 & \quad EMBO/negative\_sampling\_pmb & 0.2033 \\
\quad annakotarba/sentence-similarity & 0.0685 & \quad kamalkraj/BioSimCSE-BioLinkBERT-BASE & 0.2221 \\
\quad s-t/paraphrase-multilingual-MiniLM-L12-v2 & 0.0685 & \quad menadsa/S-PubMedBERT & 0.2610 \\
\quad DataikuNLP/paraphrase-multilingual-MiniLM-L12-v2 & 0.0685 & \quad andreinsardi/SciBERT-SolarPhysics-Search & 0.2633 \\
\quad vahoaka/sentence-transformers-model-vahoaka-v1 & 0.0732 & \quad xlreator/snomed-biobert & 0.2667 \\
\quad tnguy564/qwen-geospatial-embedder & 0.0733 & \quad AHDMK/Sentence-BioBert-snli & 0.3038 \\
\quad s-t/paraphrase-MiniLM-L3-v2 & 0.0771 & \quad amrothemich/sapbert-sentence-transformers & 0.3222 \\
\quad gmunkhtur/paraphrase-mongolian-minilm-mn\_v2 & 0.0773 & \quad bcwarner/PubMedBERT-base-uncased-sts-combined & 0.3291 \\
\quad jaimevera1107/all-MiniLM-L6-v2-similarity-es & 0.0832 & \quad UMCU/SapBERT-from-PubMedBERT-fulltext\_bf16 & 0.3302 \\
\midrule
% --- Block 3 ---
\multicolumn{2}{l}{\textbf{Target: multilingual-e5-small}} & \multicolumn{2}{l}{\textbf{Target: mxbai-embed-large-v1}} \\
\quad beademiguelperez/s-t-multilingual-e5-small & 0.0 & \quad OrcaDB/mxbai-large & 0.0 \\
\quad d0rj/e5-small-en-ru & $5.69 e^{-4}$ & \quad WhereIsAI/UAE-Large-V1 & 0.0435 \\
\quad ferrisS/german-english-multilingual-e5-small & 0.0011 & \quad WhereIsAI/UAE-Code-Large-V1 & 0.0668 \\
\quad SergeyKarpenko1/multilingual-e5-small-legal-matryoshka\_384 & 0.0058 & \quad OrcaDB/gist-large & 0.1387 \\
\quad antoinelouis/french-me5-small & 0.0242 & \quad avsolatorio/GIST-large-Embedding-v0 & 0.1387 \\
\quad clips/e5-small-trm-nl & 0.0411 & \quad llmrails/ember-v1 & 0.1443 \\
\quad dragonkue/multilingual-e5-small-ko-v2 & 0.0566 & \quad BAAI/bge-large-en-v1.5 & 0.1447 \\
\quad djovak/embedic-small & 0.0665 & \quad katanemo/bge-large-en-v1.5 & 0.1447 \\
\quad intfloat/e5-small-v2 & 0.0783 & \quad ls-da3m0ns/bge\_large\_medical & 0.1466 \\
\quad ggrn/e5-small-v2 & 0.0783 & \quad avemio/German-RAG-UAE-LARGE-V1-TRIPLES-MERGED-HESSIAN-AI & 0.1473 \\
\midrule
% --- Block 4 ---
\multicolumn{2}{l}{\textbf{Target: multilingual-MiniLM-L12-de-en-es-fr...}} & \multicolumn{2}{l}{\textbf{Target: scincl}} \\
\quad narraticlabs/MiniLM-L6-european-union & 0.0615 & \quad andreinsardi/SciBERT-SolarPhysics-Search & 0.3285 \\
\quad sergeyzh/rubert-mini-frida & 0.1149 & \quad nasa-impact/nasa-ibm-st.38m & 0.3294 \\
\quad gmunkhtur/paraphrase-mongolian-minilm-mn\_v2 & 0.1302 & \quad ibm-granite/granite-embedding-125m-english & 0.3309 \\
\quad Omartificial-Intelligence-Space/Arabic-MiniLM-L12-v2... & 0.1303 & \quad s-t/allenai-specter & 0.3360 \\
\quad jinaai/jina-embedding-t-en-v1 & 0.1343 & \quad prdev/mini-gte & 0.3365 \\
\quad vahoaka/sentence-transformers-model-vahoaka-v1 & 0.1378 & \quad Sampath1987/EnergyEmbed-v1 & 0.3370 \\
\quad annakotarba/sentence-similarity & 0.1388 & \quad jinaai/jina-embedding-b-en-v1 & 0.3411 \\
\quad s-t/paraphrase-multilingual-MiniLM-L12-v2 & 0.1388 & \quad hkunlp/instructor-base & 0.3428 \\
\quad DataikuNLP/paraphrase-multilingual-MiniLM-L12-v2 & 0.1388 & \quad MongoDB/mdbr-leaf-mt & 0.3438 \\
\quad easonanalytica/cnm-multilingual-small-v2 & 0.1402 & \quad nasa-impact/nasa-smd-ibm-st-v2 & 0.3458 \\
\midrule
% --- Block 5 ---
\multicolumn{2}{l}{\textbf{Target: GIST-small-Embedding-v0}} & \multicolumn{2}{l}{\textbf{Target: bge-large-en-v1.5}} \\
\quad OrcaDB/gte-small & 0.0494 & \quad BAAI/bge-large-en-v1.5 & 0.0 \\
\quad embaas/s-t-gte-small & 0.0494 & \quad llmrails/ember-v1 & 0.0465 \\
\quad thenlper/gte-small & 0.0494 & \quad WhereIsAI/UAE-Large-V1 & 0.1284 \\
\quad MoralHazard/NSFW-GIST-small & 0.0698 & \quad ls-da3m0ns/bge\_large\_medical & 0.1290 \\
\quad JALLAJ/5epo & 0.0705 & \quad OrcaDB/mxbai-large & 0.1447 \\
\quad austinpatrickm/finetuned\_bge\_embeddings\_v5\_small\_v1.5 & 0.0742 & \quad mixedbread-ai/mxbai-embed-large-v1 & 0.1447 \\
\quad SmartComponents/bge-micro-v2 & 0.0763 & \quad WhereIsAI/UAE-Code-Large-V1 & 0.1520 \\
\quad TaylorAI/bge-micro-v2 & 0.0763 & \quad avemio/German-RAG-UAE-LARGE-V1-TRIPLES-MERGED-HESSIAN-AI & 0.1820 \\
\quad TaylorAI/gte-tiny & 0.0767 & \quad sdadas/mmlw-retrieval-roberta-large & 0.1881 \\
\quad intfloat/e5-small & 0.0784 & \quad sdadas/mmlw-retrieval-e5-large & 0.1883 \\
\midrule
% --- Block 6 ---
\multicolumn{2}{l}{\textbf{Target: bge-base-en-v1.5}} & \multicolumn{2}{l}{\textbf{Target: Qwen3-Embedding-0.6B}} \\
\quad Shashwat13333/bge-base-en-v1.5\_v4 & 0.0 & \quad woodx/Qwen3-Embedding-0.6B-SGLang & 0.0 \\
\quad OrcaDB/bge-base & 0.0 & \quad michaelfeil/Qwen3-Embedding-0.6B-auto & 0.0 \\
\quad datasocietyco/bge-base-en-v1.5-course-recommender-v5 & 0.0138 & \quad DeepMount00/Ita-Search & 0.2210 \\
\quad axondendriteplus/Legal-Embed-bge-base-en-v1.5 & 0.0302 & \quad tomaarsen/Qwen3-Embedding-0.6B-18-layers & 0.2282 \\
\quad pavanmantha/bge-base-en-bioembed & 0.0400 & \quad Tarka-AIR/Tarka-Embedding-350M-V1 & 0.2519 \\
\quad potsu-potsu/bge-base-biomedical-matryoshka-v3 & 0.0655 & \quad billatsectorflow/stella\_en\_400M\_v5 & 0.2638 \\
\quad prashpathak/xlscout\_standigger\_2\_aug & 0.0792 & \quad nlpai-lab/KURE-v1 & 0.2664 \\
\quad iris49/3gpp-embedding-model-v0 & 0.0934 & \quad dragonkue/BGE-m3-ko & 0.2667 \\
\quad infgrad/stella-base-en-v2 & 0.1284 & \quad seroe/Qwen3-Embedding-0.6B-turkish-triplet-matryoshka-v2 & 0.2688 \\
\quad avsolatorio/GIST-Embedding-v0 & 0.1291 & \quad dragonkue/snowflake-arctic-embed-l-v2.0-ko & 0.2700 \\
\midrule
% --- Block 7 ---
\multicolumn{2}{l}{\textbf{Target: all-mpnet-base-v2}} & \multicolumn{2}{l}{\textbf{Target: gte-large-en-v1.5}} \\
\quad optimum-intel-internal-testing/all-mpnet-base-v2 & 0.0 & \quad Salesforce/SFR-Embedding-Code-400M\_R & 0.1436 \\
\quad flax-sentence-embeddings/all\_datasets\_v4\_mpnet-base & $3.22 e^{-4}$ & \quad rbhatia46/financial-rag-matryoshka & 0.1894 \\
\quad spartan8806/atles-champion-embedding & 0.0429 & \quad dpanea/skill-assignment-transformer & 0.2376 \\
\quad wwydmanski/all-mpnet-base-v2-legal-v0.1 & 0.0440 & \quad thenlper/gte-large & 0.3088 \\
\quad hojzas/setfit-proj8-all-mpnet-base-v2 & 0.0628 & \quad OrcaDB/gist-large & 0.3143 \\
\quad deepset/all-mpnet-base-v2-table & 0.0897 & \quad avsolatorio/GIST-large-Embedding-v0 & 0.3143 \\
\quad vaios-stergio/all-mpnet-base-v2-dblp-aminer-50k-triplets-64 & 0.1043 & \quad jinaai/jina-embedding-l-en-v1 & 0.3196 \\
\quad bwang0911/jev2-legal & 0.1053 & \quad flax-sentence-embeddings/all\_datasets\_v3\_roberta-large & 0.3198 \\
\quad s-t/all-mpnet-base-v1 & 0.1164 & \quad s-t/all-roberta-large-v1 & 0.3198 \\
\quad flax-sentence-embeddings/all\_datasets\_v3\_mpnet-base & 0.1165 & \quad HIT-TMG/KaLM-embedding-multilingual-mini-v1 & 0.3227 \\
\midrule
% --- Block 8 ---
\multicolumn{2}{l}{\textbf{Target: nomic-embed-text-v1.5}} & \multicolumn{2}{l}{\textbf{Target: GradientBoosting\_model\_5\_samples...}} \\
\quad asmud/nomic-embed-indonesian & 0.0048 & \quad sergeyzh/rubert-mini-frida & 0.6257 \\
\quad corto-ai/nomic-embed-text-v1 & 0.0770 & \quad jinaai/jina-embedding-t-en-v1 & 0.6258 \\
\quad nomic-ai/nomic-embed-text-v1 & 0.0770 & \quad cnmoro/custom-model2vec-tokenlearn-small & 0.6270 \\
\quad nomic-ai/nomic-embed-text-v1-ablated & 0.0987 & \quad cointegrated/rubert-tiny2 & 0.6271 \\
\quad CatSchroedinger/nomic-v1.5-financial-matryoshka & 0.0990 & \quad sergeyzh/rubert-mini-sts & 0.6281 \\
\quad nomic-ai/nomic-embed-text-v1-unsupervised & 0.1350 & \quad narraticlabs/MiniLM-L6-european-union & 0.6305 \\
\quad simonosgoode/nomic\_embed\_fine\_tune\_law\_1.5 & 0.1400 & \quad minishlab/potion-multilingual-128M & 0.6311 \\
\quad ValentinaKim/bge-base-automobile-matryoshka & 0.1617 & \quad newmindai/TurkEmbed4STS-Static & 0.6313 \\
\quad intfloat/e5-base & 0.1662 & \quad h4g3n/multilingual-MiniLM-L12-de-en-es-fr-it-nl-pl-pt & 0.6313 \\
\quad KarBik/legal-french-matroshka & 0.1737 & \quad Samizie/avia-MiniLM-L12-v2 & 0.6316 \\

\bottomrule
\end{tabular}
}
\caption{Nearest neighbours for 16 encoders (Top-10 each) sorted in ascending order of pairwise $\ell_1$ values.
s-t denotes \textit{sentence-transformers}. 0.0 indicates value is effectively zero for 12 decimals.}
\label{tab:nearest_neighbors_16_10}
\end{table*}

\section{Heirrarchical Clustering of Nearest Neighbours}
\label{sec:dendrogram_NN}

We draw a dendrogram for 10 randomly-selected encoders from the 30 encoders for computing nearest neighbours with the top 5 nearest neighbours for each to visualise their connections, shown in \autoref{fig:dendrogram}.
It is clear that all the neighbourhoods based on different core sentence encoders shown in the same colour are separable with neighbours close to each other.
This validates that our map captures both the local and glocal relationships between encoders.

\begin{figure*}[t!]
    \centering
    \includegraphics[width=1\linewidth]{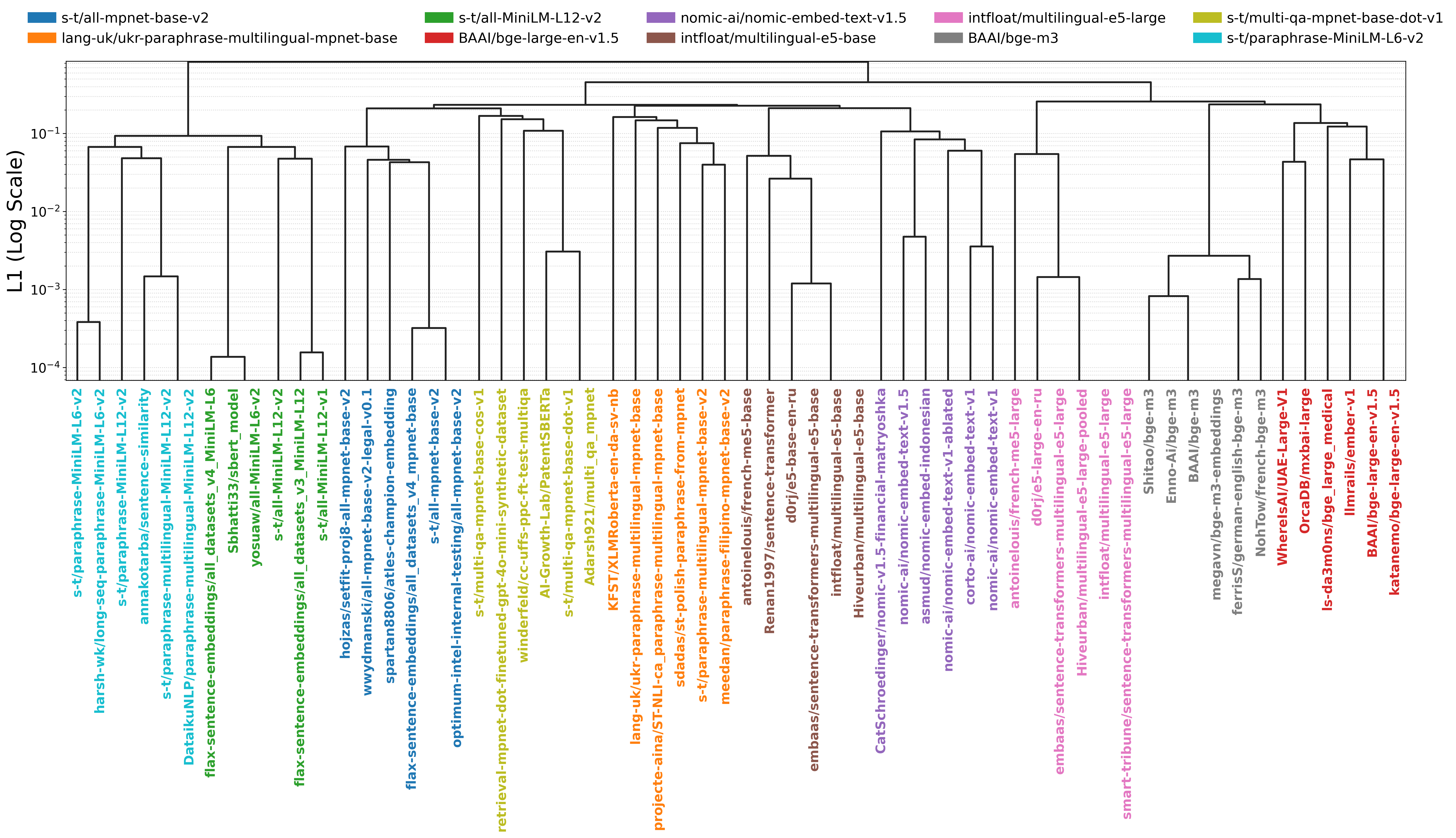}
    \caption{Heirrarchical Clustering for 10 randomly selected encoders with top 5 nearest neighbours.}
    \label{fig:dendrogram}
\end{figure*}

\section{More Maps by Training Datasets and Dimensionalities}
\label{sec:more_map_class}

\autoref{fig:more_map} shows two maps based on training datasets and the dimensionality of the embeddings.

For the map coloured by the training datasets~\autoref{fig:map_dataset}, the encoders trained with multiple datasets tend to be located across the map.
The encoders trained with single datasets exist more in the left areas of the map

For the map coloured by the output dimensionality of the embeddings corresponding to the encoder~\autoref{fig:map_dim}, we see a pattern that encoders with lower dimensionality (e.g. 384) are more located in the upper-middle area of the map, with similar size groups of 128 and 512 close together.
Encoders with higher dimensionality of 1024 have clear groups on the right and bottom side in the map.
Dimensinality of 768 is the most common dimensionality for sentence encoders, where encoders with dimensionality of 768 are spread out in the map.

\begin{figure}[t]
    \centering
    
    % --- Subfigure 1: Datasets ---
    \begin{subfigure}[b]{\linewidth}
        \centering
        \includegraphics[width=\linewidth]{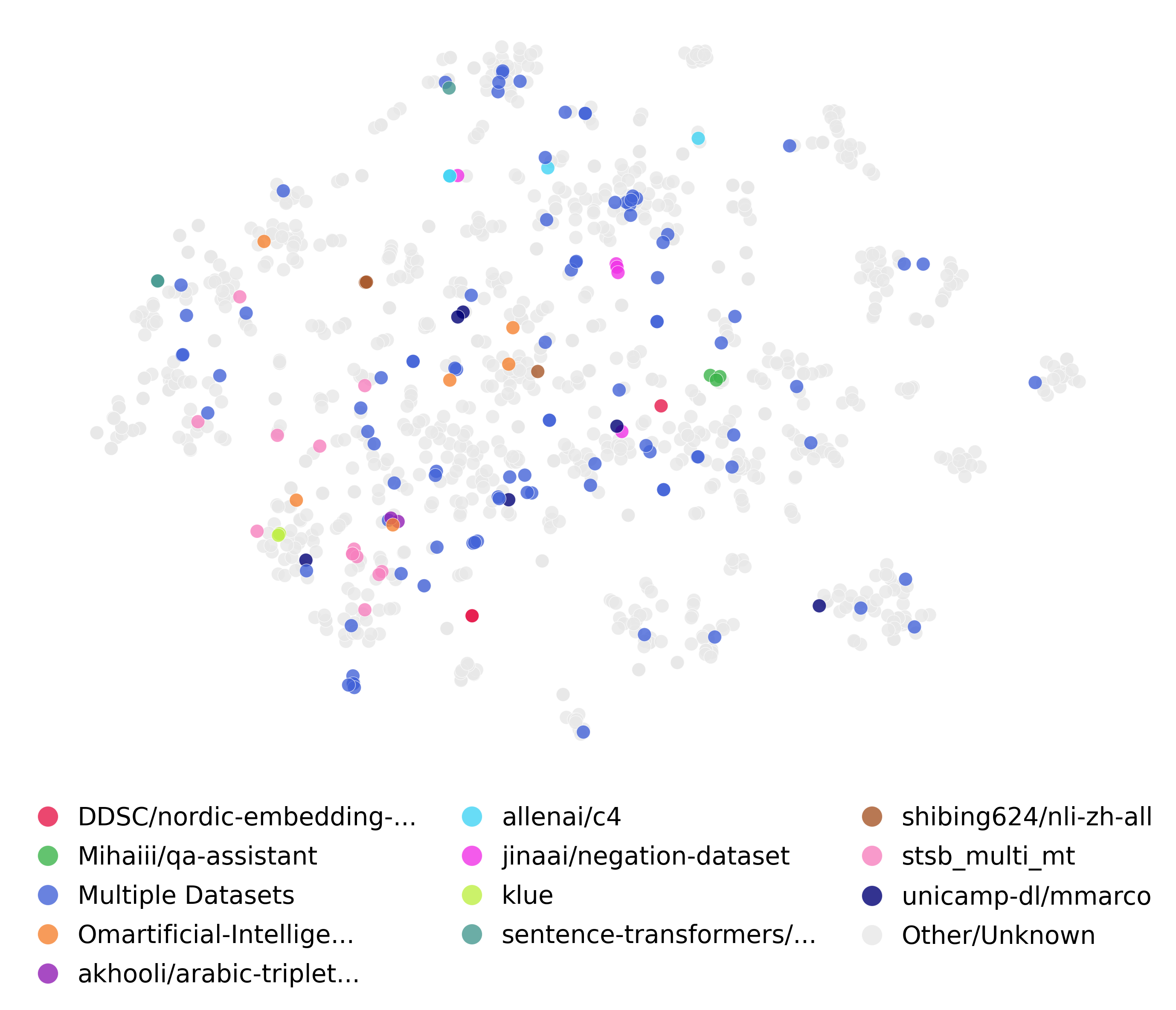}
        \caption{Training Datasets}
        \label{fig:map_dataset}
    \end{subfigure}
    
    \vspace{1em} 
    
    % --- Subfigure 2: Dimensionality ---
    \begin{subfigure}[b]{\linewidth}
        \centering
        \includegraphics[width=\linewidth]{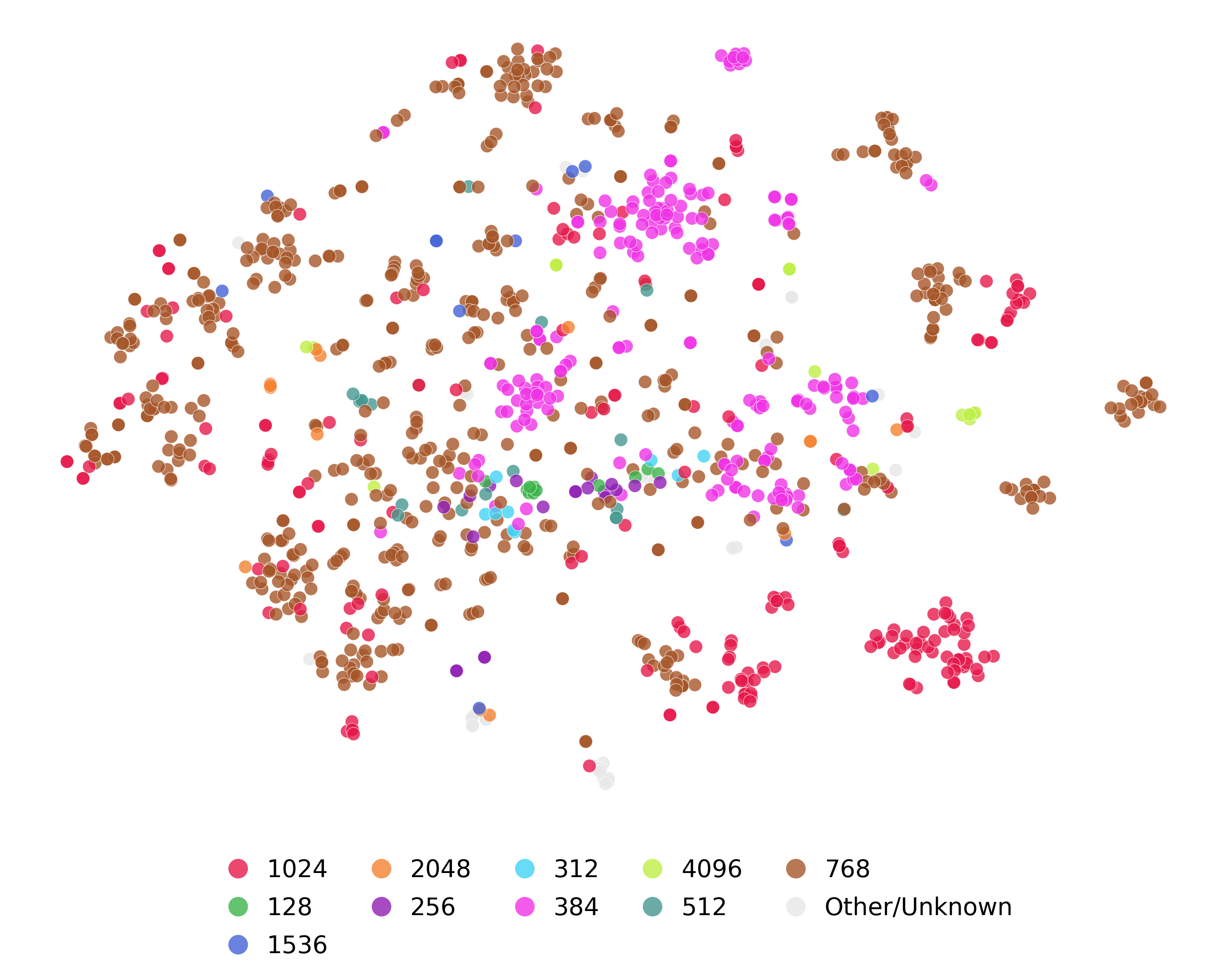}
        \caption{Dimensionality}
        \label{fig:map_dim}
    \end{subfigure}
    
    \caption{Maps based on training datasets and dimensionality of encoders.}
    \label{fig:more_map}
\end{figure}

\section{Zoomed in Heirrarchical Clustering for Top 100 Most-downloaded Encoders}
\label{sec:zoomin_den}

\autoref{fig:dendrogram_top_100_zoom} shows the zoomed-in dendrogram for hierarchical clustering of the top 100 most-downloaded encoders.

\begin{figure*}[t!]
    \centering
    \includegraphics[width=1\linewidth]{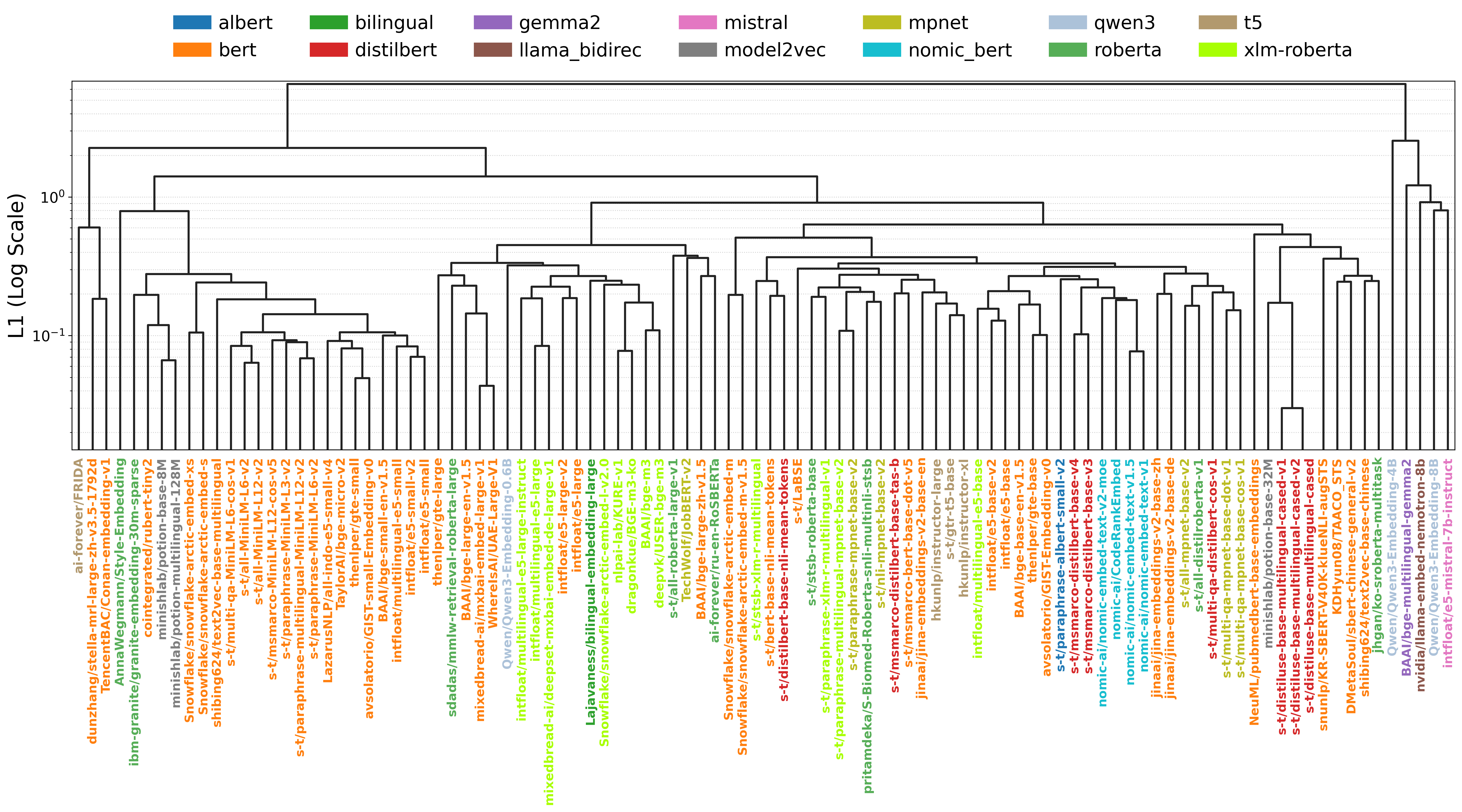}
    \caption{Zoomed-in hierarchical clustering of the top 100 most-downloaded encoders, coloured by model type. $\ell_1$ values are reported in log scale for better visualisation.}
    \label{fig:dendrogram_top_100_zoom}
\end{figure*}

\section{Full Encoder List for the Map}
\label{sec:full_model_list}

\autoref{tab:mteb_full_model_list_metadata} provides the list of 1101 encoders with encoder type, parameter size and dimensionality, which are three attributes used in this paper.

\section{Evaluating the Sentence Set Selection}
\label{sec:select_10000}

To validate selecting 10,000 sentences for creating high-quality maps, we further create feature matrices and maps of encoders based on 1000, 2500 and 5000 sentences as the sentence set $\cS$ in~\autoref{sec:method}.

As shown in \autoref{fig:5k_maps}, maps created with 5000 sentences has comparable and clear patterns in the maps as those created with 10,000 sentences in the paper.
For example, for the map coloured by encoder type, which is the most complicated map, has clear samll groups such as model2vec and gemma3\_text same as the map created with 10,000 sentences.
This validates that our selected sentence set has high quality for generating sentence embeddings and thus the maps of encoders.

Additionally, we test the effectiveness of the sentence set sizes of 1000, 2500, 5000 and 10000 on the average Spearman Correlation between the true and predicted performance on 68 MTEB tasks, as shown in~\autoref{fig:Spearman_MTEB_size}.
The average Spearman correlation saturates for a sentence set of 2500 sentences, indicating that 2500 sentences is sufficient to obtain peak average correlation on MTEB and our selection of 10,000 sentences can accurately capture the embedding spaces of all target encoders.

\begin{figure}[t!]
    \centering
    \includegraphics[width=1\linewidth]{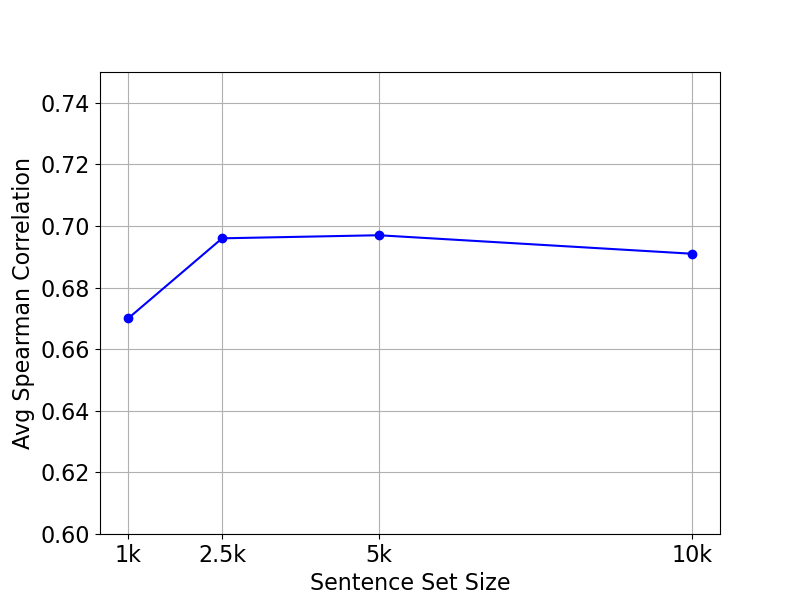}
    \caption{Average Spearman Correlation between the true and predicted performance on 68 MTEB tasks, based on sentence set sizes.}
    \label{fig:Spearman_MTEB_size}
\end{figure}

\begin{figure*}[t]
    \centering
    % Row 1
    \begin{subfigure}[b]{0.48\textwidth}
        \centering
        \includegraphics[width=\textwidth]{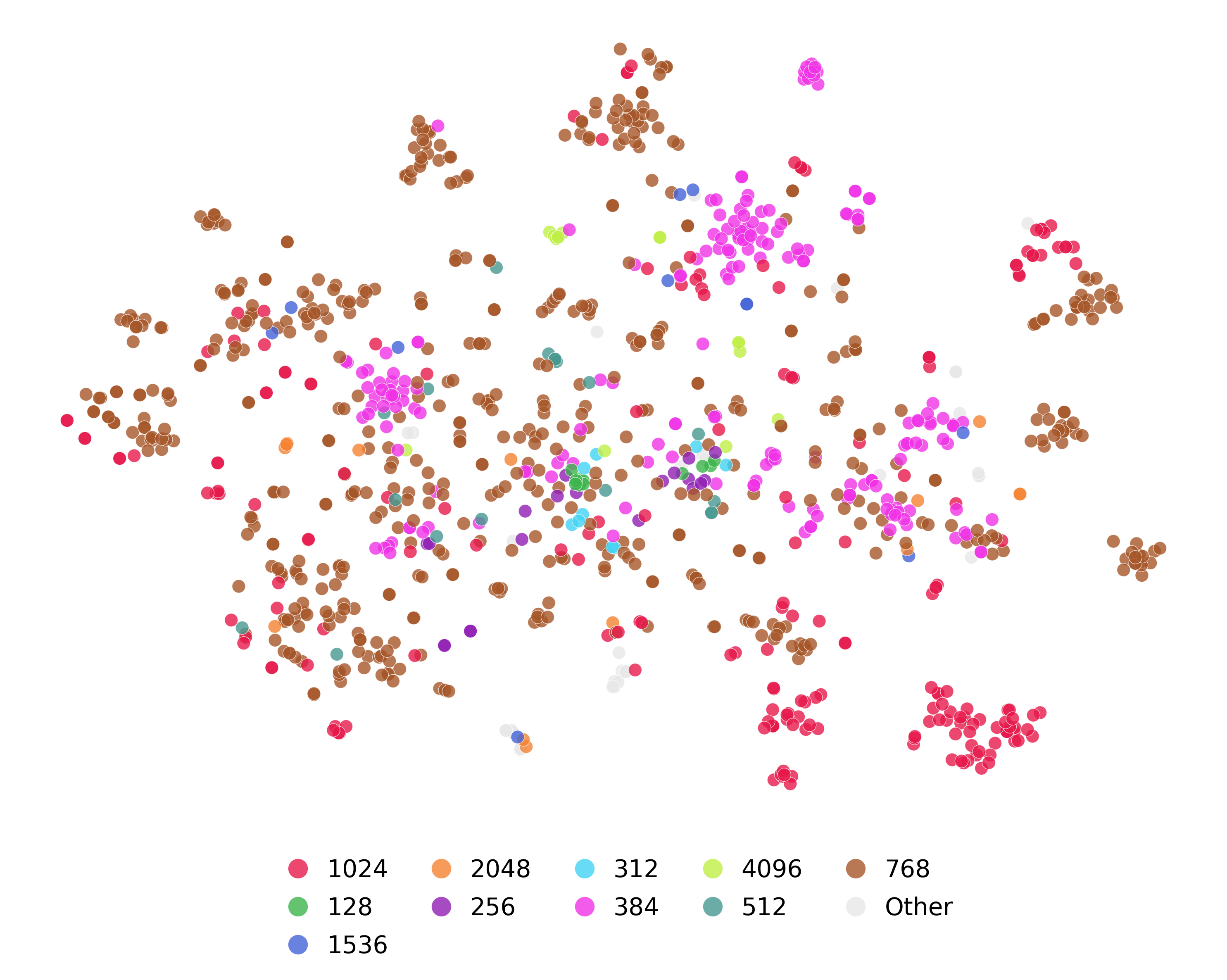}
        \caption{Dimensionality}
    \end{subfigure}
    \hfill
    \begin{subfigure}[b]{0.48\textwidth}
        \centering
        \includegraphics[width=\textwidth]{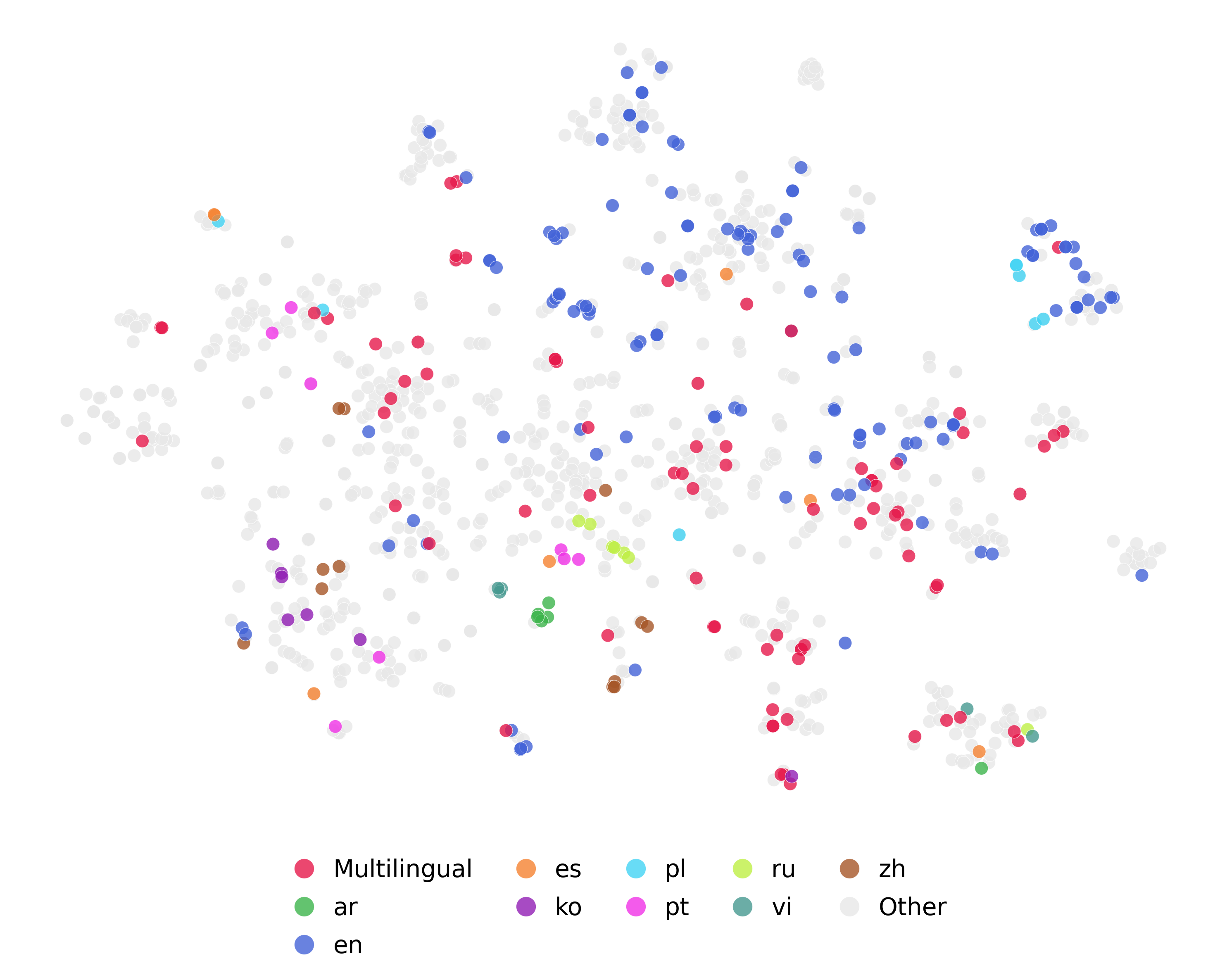}
        \caption{Pre-trained Language}
    \end{subfigure}

    \vspace{1em} % Vertical spacing between rows

    % Row 2
    \begin{subfigure}[b]{0.48\textwidth}
        \centering
        \includegraphics[width=\textwidth]{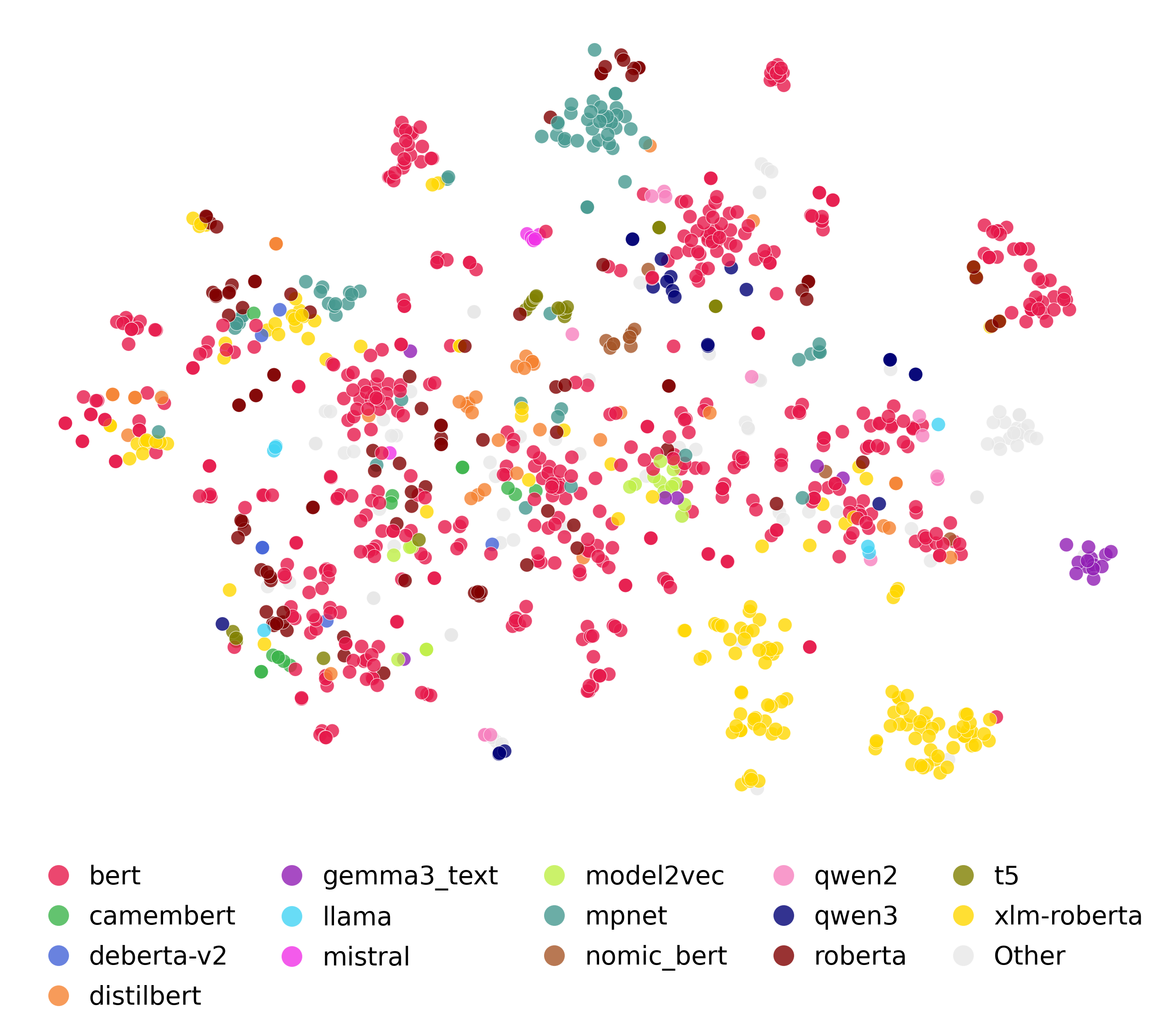}
        \caption{Encoder Type}
    \end{subfigure}
    \hfill
    \begin{subfigure}[b]{0.48\textwidth}
        \centering
        \includegraphics[width=\textwidth]{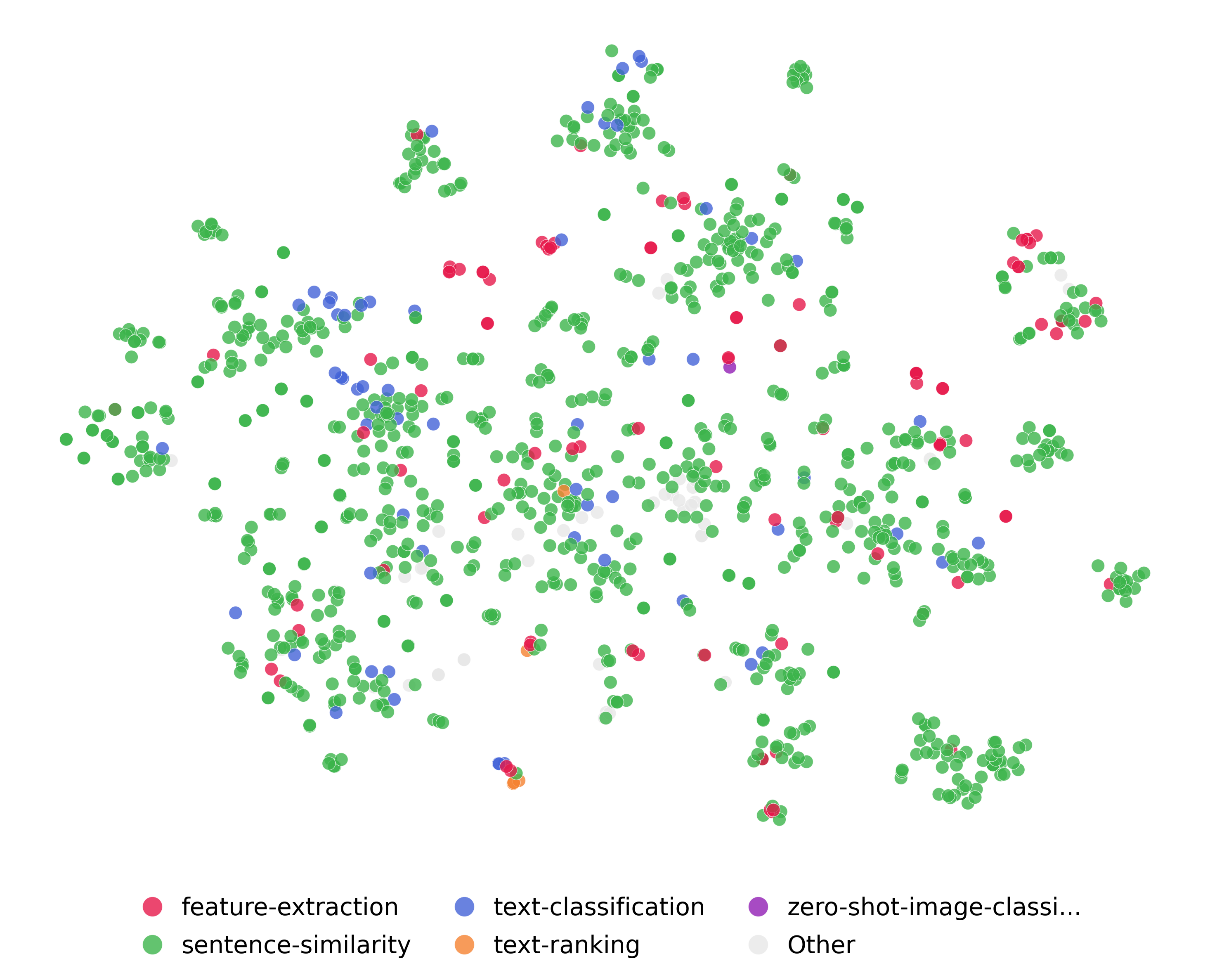}
        \caption{Training Task}
    \end{subfigure}

    \vspace{1em}

    % Row 3
    \begin{subfigure}[b]{0.48\textwidth}
        \centering
        \includegraphics[width=\textwidth]{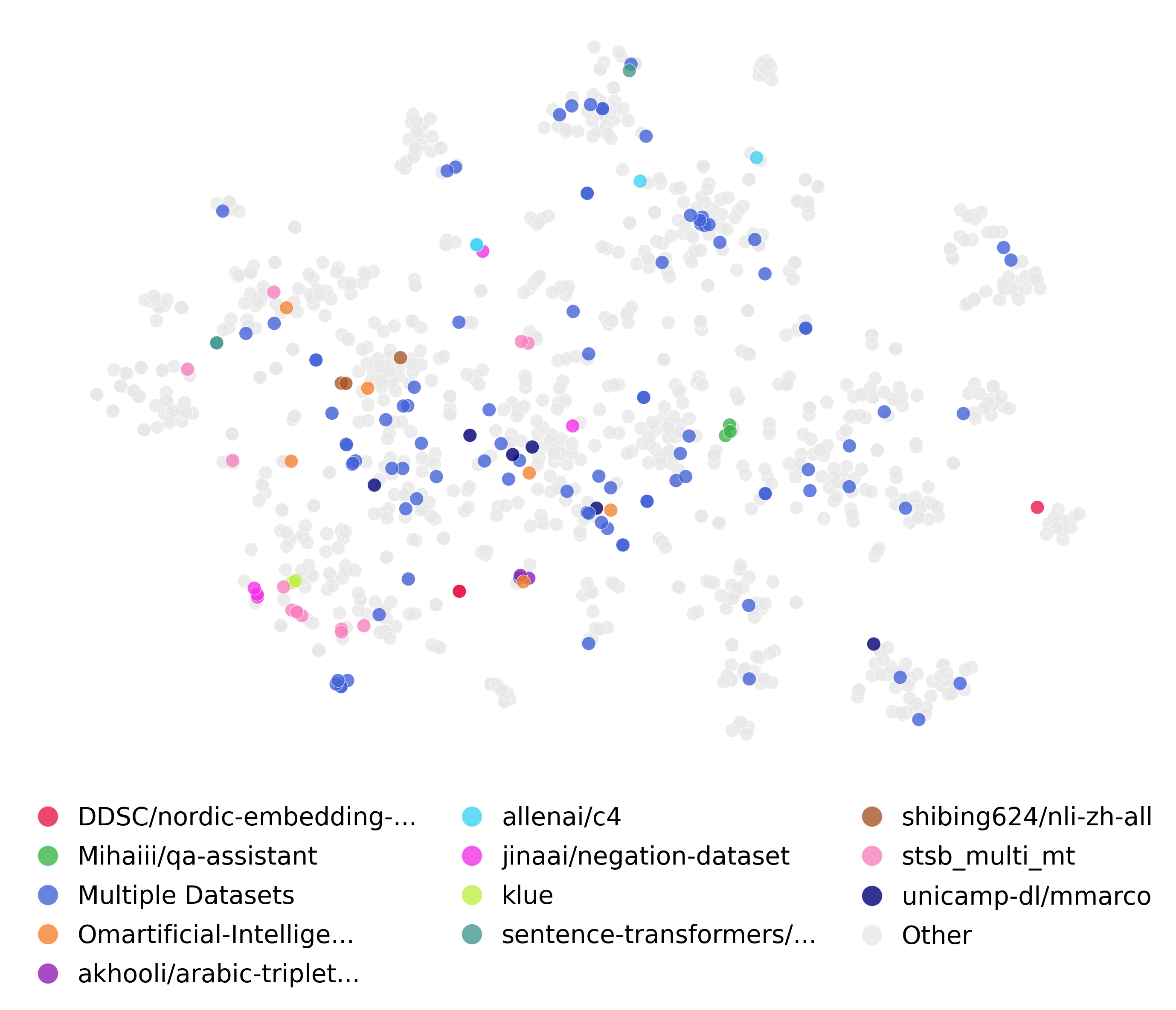}
        \caption{Training Dataset}
        \label{fig:clusters_dataset}
    \end{subfigure}
    \hfill
    \begin{subfigure}[b]{0.48\textwidth}
        \centering
        \includegraphics[width=\textwidth]{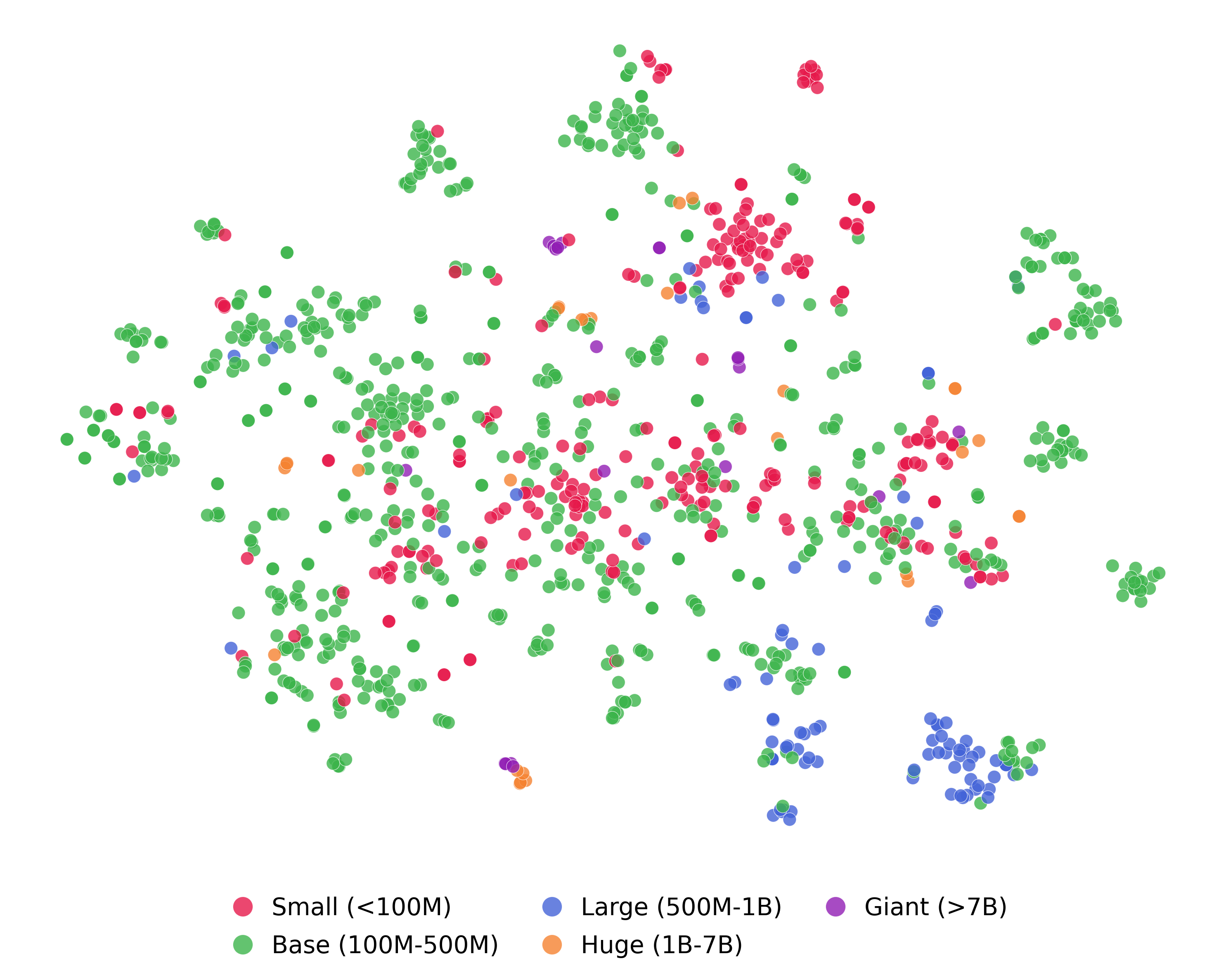}
        \caption{Enocder Parameter Size}
    \end{subfigure}

    \caption{Maps created with a sentence set of size 5000.}
    \label{fig:5k_maps}
\end{figure*}

\section{Full Encoder List for MTEB results}
\label{sec:full_encoder_list_MTEB}

\autoref{tab:mteb_full_model_list_metadata} shows the full list of 112 encoders with parameter size and dimensionality.

\clearpage
\onecolumn
\begingroup
\small
\setlength{\tabcolsep}{6pt}
\setlength{\LTcapwidth}{\textwidth}

\begin{longtable}{c p{8cm} r c}
\toprule
\textbf{\#} & \textbf{Encoder Name} & \textbf{\# Params} & \textbf{Dimensionality} \\
\midrule
\endfirsthead

\toprule
\textbf{\#} & \textbf{Encoder Name} & \textbf{\# Params} & \textbf{Dimensionality} \\
\midrule
\endhead

\bottomrule
\endfoot

\bottomrule
\noalign{\vspace{1ex}}
\caption{Full list of 112 encoders used in MTEB evaluation with parameter size and dimensionality.}
\label{tab:mteb_full_model_list_metadata} \\
\endlastfoot

1 & Alibaba-NLP/gte-base-en-v1.5 & 136.8M & 768 \\
2 & Alibaba-NLP/gte-large-en-v1.5 & 434.1M & 1024 \\
3 & Alibaba-NLP/gte-multilingual-base & 305.4M & 768 \\
4 & BAAI/bge-base-en-v1.5 & 109.5M & 768 \\
5 & BAAI/bge-large-en-v1.5 & 335.1M & 1024 \\
6 & BAAI/bge-small-en-v1.5 & 33.4M & 384 \\
7 & Gameselo/STS-multilingual-mpnet-base-v2 & 278.0M & 768 \\
8 & HIT-TMG/KaLM-embedding-multilingual-mini-instruct-v1 & 494.0M & 896 \\
9 & HIT-TMG/KaLM-embedding-multilingual-mini-instruct-v1.5 & 494.0M & 896 \\
10 & HIT-TMG/KaLM-embedding-multilingual-mini-v1 & 494.0M & 896 \\
11 & Lajavaness/bilingual-embedding-base & 278.0M & 768 \\
12 & Lajavaness/bilingual-embedding-large & 559.9M & 1024 \\
13 & Lajavaness/bilingual-embedding-small & 117.7M & 384 \\
14 & Linq-AI-Research/Linq-Embed-Mistral & 7.11B & 4096 \\
15 & Mihaiii/Bulbasaur & 17.4M & 384 \\
16 & Mihaiii/Ivysaur & 22.7M & 384 \\
17 & Mihaiii/Venusaur & 15.6M & 384 \\
18 & Mihaiii/gte-micro-v4 & 19.2M & 384 \\
19 & SGPT-1.3B-weightedmean-msmarco-specb-bitfit & 1.34B (Est) & 2048 \\
20 & SGPT-125M-weightedmean-msmarco-specb-bitfit & 137.8M (Est) & 768 \\
21 & SGPT-125M-weightedmean-nli-bitfit & 137.8M (Est) & 768 \\
22 & SGPT-5.8B-weightedmean-msmarco-specb-bitfit & 5.87B (Est) & 4096 \\
23 & Omartificial/Arabert-all-nli-triplet-Matryoshka & 135.2M & 768 \\
24 & Omartificial/Arabic-MiniLM-L12-v2-all-nli-triplet & 117.7M & 384 \\
25 & Omartificia/Arabic-all-nli-triplet-Matryoshka & 278.0M & 768 \\
26 & Omartificial/Arabic-labse-Matryoshka & 470.9M & 768 \\
27 & Omartificial/Arabic-mpnet-base-all-nli-triplet & 109.5M & 768 \\
28 & Omartificial/Marbert-all-nli-triplet-Matryoshka & 162.8M & 768 \\
29 & OrcaDB/cde-small-v1 & 281.1M & 768 \\
30 & OrcaDB/gte-base-en-v1.5 & 136.8M & 768 \\
31 & Salesforce/SFR-Embedding-2\_R & 7.11B & 4096 \\
32 & Salesforce/SFR-Embedding-Mistral & 7.11B & 4096 \\
33 & SmartComponents/bge-micro-v2 & 8.7M (Est) & 384 \\
34 & Snowflake/snowflake-arctic-embed-l & 334.1M & 1024 \\
35 & Snowflake/snowflake-arctic-embed-l-v2.0 & 567.8M & 1024 \\
36 & Snowflake/snowflake-arctic-embed-m & 108.9M & 768 \\
37 & Snowflake/snowflake-arctic-embed-m-long & 136.7M & 768 \\
38 & Snowflake/snowflake-arctic-embed-m-v1.5 & 108.9M & 768 \\
39 & Snowflake/snowflake-arctic-embed-s & 33.2M & 384 \\
40 & Snowflake/snowflake-arctic-embed-xs & 22.6M & 384 \\
41 & TaylorAI/bge-micro & 17.4M & 384 \\
42 & TaylorAI/bge-micro-v2 & 17.4M & 384 \\
43 & TaylorAI/gte-tiny & 22.7M & 384 \\
44 & WhereIsAI/UAE-Large-V1 & 335.1M & 1024 \\
45 & aari1995/German\_Semantic\_STS\_V2 & 335.7M & 1024 \\
46 & abhinand/MedEmbed-small-v0.1 & 33.4M & 384 \\
47 & ai-forever/ru-en-RoSBERTa & 403.7M & 1024 \\
48 & arkohut/jina-embeddings-v2-base-en & 137.4M & 768 \\
49 & avsolatorio/GIST-Embedding-v0 & 109.5M & 768 \\
50 & avsolatorio/GIST-all-MiniLM-L6-v2 & 22.7M & 384 \\
51 & avsolatorio/GIST-large-Embedding-v0 & 335.1M & 1024 \\
52 & avsolatorio/GIST-small-Embedding-v0 & 33.4M & 384 \\
53 & bigscience/sgpt-bloom-7b1-msmarco & 7.07B (Est) & 4096 \\
54 & cointegrated/rubert-tiny2 & 29.4M & 312 \\
55 & corto-ai/nomic-embed-text-v1 & 136.7M & 768 \\
56 & deepvk/USER-base & 124.0M & 768 \\
57 & ggrn/e5-small-v2 & 33.4M (Est) & 384 \\
58 & hkunlp/instructor-base & 110.2M (Est) & 768 \\
59 & hkunlp/instructor-large & 335.7M (Est) & 768 \\
60 & hkunlp/instructor-xl & 1.24B (Est) & 768 \\
61 & ibm-granite/granite-embedding-125m-english & 124.6M & 768 \\
62 & ibm-granite/granite-embedding-278m-multilingual & 278.0M & 768 \\
63 & ibm-granite/granite-embedding-30m-english & 30.3M & 384 \\
64 & infgrad/stella-base-en-v2 & 54.8M (Est) & 768 \\
65 & intfloat/e5-base & 109.5M & 768 \\
66 & intfloat/e5-base-v2 & 109.5M & 768 \\
67 & intfloat/e5-large & 335.1M & 1024 \\
68 & intfloat/e5-large-v2 & 335.1M & 1024 \\
69 & intfloat/e5-mistral-7b-instruct & 7.11B & 4096 \\
70 & intfloat/e5-small & 33.4M & 384 \\
71 & intfloat/e5-small-v2 & 33.4M & 384 \\
72 & intfloat/multilingual-e5-base & 278.0M & 768 \\
73 & intfloat/multilingual-e5-large & 559.9M & 1024 \\
74 & intfloat/multilingual-e5-large-instruct & 559.9M & 1024 \\
75 & intfloat/multilingual-e5-small & 117.7M & 384 \\
76 & jinaai/jina-embedding-b-en-v1 & 109.6M (Est) & 768 \\
77 & jinaai/jina-embedding-l-en-v1 & 335.0M (Est) & 1024 \\
78 & jinaai/jina-embedding-s-en-v1 & 35.3M (Est) & 512 \\
79 & jinaai/jina-embeddings-v2-base-en & 137.4M & 768 \\
80 & jinaai/jina-embeddings-v2-small-en & 32.7M & 512 \\
81 & jxm/cde-small-v1 & 281.1M & 768 \\
82 & katanemo/bge-large-en-v1.5 & 335.1M & 1024 \\
83 & khoa-klaytn/bge-base-en-v1.5-angle & 109.5M & 768 \\
84 & liddlefish/privacy\_embedding\_rag\_10k\_base\_15\_final & 109.5M & 768 \\
85 & llmrails/ember-v1 & 335.1M & 1024 \\
86 & minishlab/M2V\_base\_output & 7.6M & 256 \\
87 & minishlab/potion-base-2M & 1.9M & 64 \\
88 & minishlab/potion-base-4M & 3.8M & 128 \\
89 & minishlab/potion-base-8M & 7.6M & 256 \\
90 & mixedbread-ai/mxbai-embed-2d-large-v1 & 335.1M & 1024 \\
91 & mixedbread-ai/mxbai-embed-large-v1 & 335.1M & 1024 \\
92 & nomic-ai/nomic-embed-text-v1 & 136.7M & 768 \\
93 & nomic-ai/nomic-embed-text-v1-ablated & 136.7M (Est) & 768 \\
94 & nomic-ai/nomic-embed-text-v1-unsupervised & 136.7M (Est) & 768 \\
95 & nomic-ai/nomic-embed-text-v1.5 & 136.7M & 768 \\
96 & sdadas/mmlw-e5-base & 278.0M & 768 \\
97 & sdadas/mmlw-e5-large & 559.9M & 1024 \\
98 & sdadas/mmlw-e5-small & 117.7M & 384 \\
99 & sdadas/mmlw-roberta-base & 124.4M & 768 \\
100 & sdadas/mmlw-roberta-large & 435.0M & 1024 \\
101 & sentence-transformers/LaBSE & 470.9M & 768 \\
102 & sentence-transformers/all-MiniLM-L12-v2 & 33.4M & 384 \\
103 & sentence-transformers/all-MiniLM-L6-v2 & 22.7M & 384 \\
104 & sentence-transformers/all-mpnet-base-v2 & 109.5M & 768 \\
105 & paraphrase-multilingual-MiniLM-L12-v2 & 117.7M & 384 \\
106 & paraphrase-multilingual-mpnet-base-v2 & 278.0M & 768 \\
107 & sergeyzh/LaBSE-ru-turbo & 128.3M & 768 \\
108 & sergeyzh/rubert-tiny-turbo & 29.2M & 312 \\
109 & shibing624/text2vec-base-multilingual & 117.7M & 384 \\
110 & thenlper/gte-base & 109.5M & 768 \\
111 & thenlper/gte-large & 335.1M & 1024 \\
112 & thenlper/gte-small & 33.4M & 384 \\

\end{longtable}
\endgroup
\twocolumn

\section{Implementation Details for Spearman Correlation between True and Predicted MTEB Performance}
\label{sec:impletation_details_mteb}

We conduct the performance prediction using elastic net regression (ElasticNetCV) from the \href{https://scikit-learn.org/stable/}{scikit-learn} library with a 5-fold cross-validation strategy. 
ElasticNetCV is linear regression with both $\ell_1$ and $\ell_2$ regularisation.
The feature vectors are first normalised using \href{https://scikit-learn.org/stable/modules/generated/sklearn.preprocessing.StandardScaler.html}{standard scaler} to ensure zero mean and unit variance.
To mitigate overfitting, given the high dimensionality of our feature vectors (10,000) relative to the number of encoders (112), we first use PCA  to project the input feature vectors from 10,000 to 50 dimensions. 
We rely on the internal 5-fold cross-validation mechanism of ElasticNetCV to automatically tune the regularisation strength $\alpha$, which is selected from a log-scale grid of 100 values starting from $\alpha_{max}$ (the smallest value penalising all coefficients to zero) down to $\alpha_{max} \times 10^{-3}$.

\section{Full Results for MTEB Correlation With Feature Vectors}
\label{sec:full_MTEB}

\autoref{tab:full_mteb_table} shows the full results of the correlation between 68 MTEB task performance and our feature vectors of 112 encoders, along with $p\text{-value}$.
We visualise task performance of the top 10 tasks~\autoref{tab:individual_task_performance} by min-max normalising the predicted and actual performance in the same figure~\autoref{fig:predicted_vs_actual}.
The Spearman correlation between the predicted and average performance is 0.779, indicating a strong correlation.
This shows our encoders in the map are connected to their downstream task performance.

\begin{figure}[t!]
    \centering
    \includegraphics[width=1\linewidth]{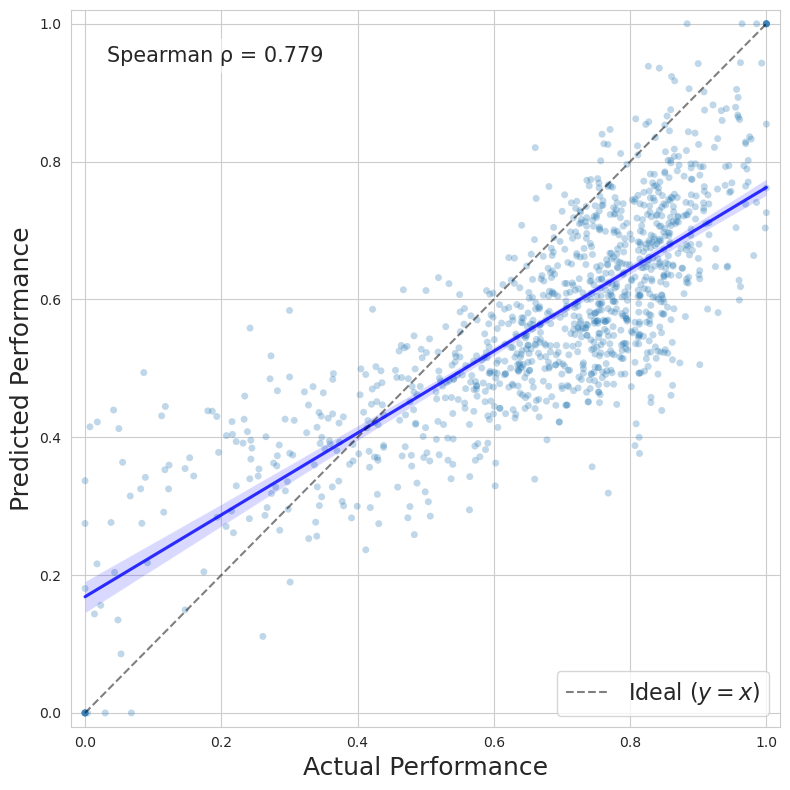}
    \caption{Predicted performance vs. actual performance for 112 encoders on top 10 tasks in~\autoref{tab:individual_task_performance}.}
    \label{fig:predicted_vs_actual}
\end{figure}

\autoref{fig:submap_QRE} visualise the 112 encoders in the map coloured by the \ac{QRE} value (sum of feature vector) for each encoder.
There are group patterns in the map, where encoders of similar \ac{QRE} values tend to be close to each other.

\begin{figure}[t!]
    
    \includegraphics[width=1\linewidth]{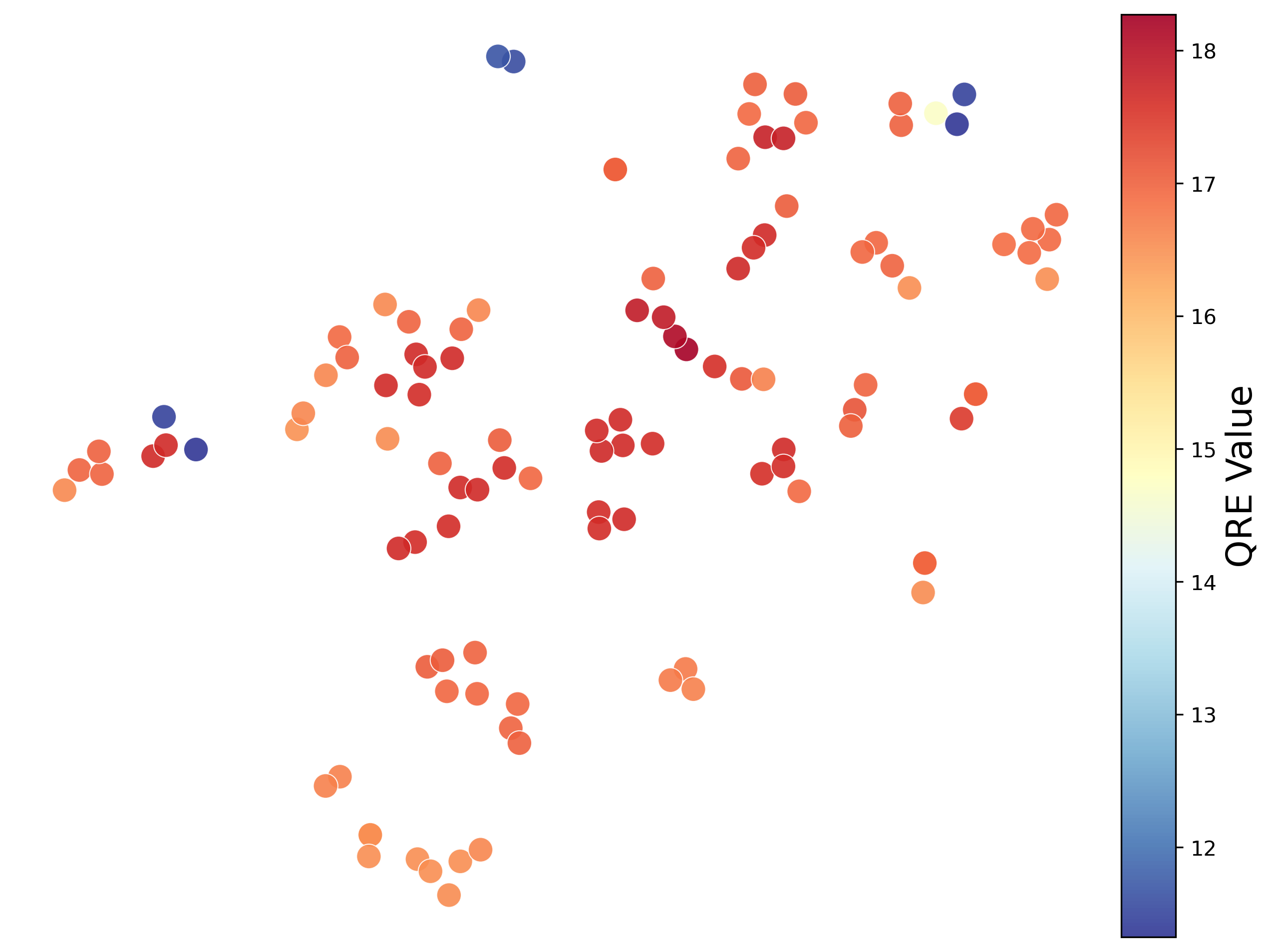}

    \caption{Submap of 112 encoders coloured by QRE values.}
    \label{fig:submap_QRE}
\end{figure}

\clearpage
\onecolumn
\begingroup
\small
\setlength{\tabcolsep}{4pt}
\setlength{\LTcapwidth}{\textwidth}

\begin{longtable}{l c c c c l}
\toprule
\textbf{Task} & \textbf{Spearman} & \textbf{$p$-value} & \textbf{Pearson} & \textbf{$p$-value} & \textbf{Task Type} \\
\midrule
\endfirsthead

\toprule
\textbf{Task} & \textbf{Spearman} & \textbf{$p$-value} & \textbf{Pearson} & \textbf{$p$-value} & \textbf{Task Type} \\
\midrule
\endhead

\bottomrule
\endfoot

\bottomrule
\noalign{\vspace{1ex}} 
\caption{Spearman and Pearson correlations between the true and predicted task performance of 68 tasks, sorted in a descending order of Spearman correlation.} \label{tab:full_mteb_table} \\
\endlastfoot

BiorxivClusteringS2S & 0.876 & $1.41 \times 10^{-36}$ & 0.837 & $1.26 \times 10^{-30}$ & Clustering \\
ArxivClusteringS2S & 0.874 & $3.22 \times 10^{-36}$ & 0.855 & $3.28 \times 10^{-33}$ & Clustering \\
ArxivClusteringP2P & 0.862 & $3.54 \times 10^{-34}$ & 0.833 & $4.29 \times 10^{-30}$ & Clustering \\
SciDocsRR & 0.861 & $4.30 \times 10^{-34}$ & 0.820 & $1.91 \times 10^{-28}$ & Reranking \\
SciFact & 0.861 & $4.48 \times 10^{-34}$ & 0.852 & $1.26 \times 10^{-32}$ & Retrieval \\
HotpotQA & 0.860 & $6.51 \times 10^{-34}$ & 0.816 & $5.48 \times 10^{-28}$ & Retrieval \\
NFCorpus & 0.857 & $2.16 \times 10^{-33}$ & 0.816 & $5.40 \times 10^{-28}$ & Retrieval \\
RedditClusteringP2P & 0.855 & $4.15 \times 10^{-33}$ & 0.838 & $1.14 \times 10^{-30}$ & Clustering \\
\end{longtable}
\endgroup
\twocolumn

\clearpage
\onecolumn
\begingroup
\small
\setlength{\tabcolsep}{4pt}
\setlength{\LTcapwidth}{\textwidth}

\begin{longtable}{c p{6.5cm} c r c}
\toprule
\textbf{\#} & \textbf{Encoder Name} & \textbf{Encoder Type} & \textbf{\# Params} & \textbf{Dimensionality} \\
\midrule
\endfirsthead

\toprule
\textbf{\#} & \textbf{Encoder Name} & \textbf{Encoder Type} & \textbf{\# Params} & \textbf{Dimensionality} \\
\midrule
\endhead

\bottomrule
\endfoot

\bottomrule
\noalign{\vspace{1ex}}
\caption{List of 1101 encoders with encoder type, parameter size and dimensionality.}
\label{tab:mteb_full_model_list_metadata} \\
\endlastfoot

1 & sentence-transformers/all-MiniLM-L6-v2 & bert & 22.7M & 384 \\
2 & sentence-transformers/all-mpnet-base-v2 & mpnet & 109.5M & 768 \\
3 & sentence-transformers/paraphrase-multilingual-MiniLM-L12-v2 & bert & 117.7M & 384 \\
4 & sentence-transformers/gtr-t5-base & t5 & 109.6M & 768 \\
5 & BAAI/bge-m3 & xlm-roberta & 567.8M (Est) & 1024 \\
6 & sentence-transformers/paraphrase-multilingual-mpnet-base-v2 & xlm-roberta & 278.0M & 768 \\
7 & sentence-transformers/multi-qa-mpnet-base-dot-v1 & mpnet & 109.5M & 768 \\
8 & Qwen/Qwen3-Embedding-0.6B & qwen3 & 595.8M & 1024 \\
\end{longtable}
\endgroup
\twocolumn

\end{document}

%% file: myacronyms.tex
\usepackage{acronym}

\usepackage{etoolbox}
\makeatletter
\newif\if@in@acrolist
\AtBeginEnvironment{acronym}{\@in@acrolisttrue}
\newrobustcmd{\LU}[2]{\if@in@acrolist#1\else#2\fi}

\newcommand{\ACF}[1]{{\@in@acrolisttrue\acf{#1}}}
\makeatother

\acrodef{MLM}[MLM]{Masked Language Model}
\acrodefplural{MLM}[MLMs]{\LU{M}{m}asked \LU{L}{l}anguage \LU{M}{m}odels}
\acrodef{LLM}[LLM]{Large Language Model}
\acrodefplural{LLM}[LLMs]{Large Language Models}
\acrodef{SoTA}[SoTA]{state-of-the-art}
\acrodef{ICL}{\LU{I}{i}n-cotext \LU{L}{l}earning}
\acrodef{SCD}{\LU{S}{s}emantic \LU{C}{c}hange \LU{D}{d}etection}
\acrodef{WiC}{Word-in-Context}
\acrodef{ITML}[ITML]{Information-Theoretic Metric Learning}
\acrodef{SDML}[SDML]{Semantic Distance Metric Learning}
\acrodef{PIP}[PIP]{Pairwise Inner Product}
\acrodef{QRE}[QRE]{Quantum Relative Entropy}
\acrodef{KL}[KL]{Kullback-Leibler}
\acrodef{SVD}[SVD]{Singular Value Decomposition}
\acrodef{PSD}[PSD]{Positive Semi-Definite}